\documentclass{article}

\usepackage{arxiv}

\usepackage[T1]{fontenc}    

\usepackage{amsmath, amsthm, amssymb, amsfonts}

\usepackage{graphicx}
\usepackage{subcaption}

\usepackage{hyperref}
\usepackage{url}

\usepackage{microtype}
\usepackage{nicefrac}
\usepackage{booktabs}
\usepackage{enumitem}

\hypersetup{
colorlinks,
citecolor = blue,
linkcolor =blue,
urlcolor = blue,
}
\usepackage{algorithm}
\usepackage{algpseudocode}
\usepackage{bbm}

\newtheorem{theorem}{Theorem}
\newtheorem{definition}[theorem]{Definition}

\newtheorem{lemma}[theorem]{Lemma}
\newtheorem{assumption}[theorem]{Assumption}
\usepackage[round]{natbib}

\def \bx{\boldsymbol{x}}
\def \bv{\boldsymbol{v}}
\def \explore{{\mathcal{T}}(t)}

\def\loss{L_{i,t}(\theta)}
\def\feature{\boldsymbol{x_{ij}}}
\def \v {V_{i}(t)}
\def \vi {V_i(t)^{-1}}

\def \et{\boldsymbol{\hat\theta}_i(t)}

\def \thetai{\boldsymbol{\theta_i}}

\def \isk{(i+s)\%K+1}

\def \dom_fixed_case{\frac{128d^2L^2\log(T)}{\kappa\Delta_{\min}^2}}
\newcommand{\ucb}{{\mathrm{UCB}}}
\newcommand{\lcb}{{\mathrm{LCB}}}

\newcommand{\alg}{ETP-GS } 
\newcommand{\improved}{IETP-GS }

\graphicspath{ {./images/} }

\title{Competing Bandits in Decentralized Contextual Matching Markets}

\author{
  \hspace{0.8cm} Satush Parikh \\
  \hspace{0.8cm} Department of Electrical Engineering\\
  \hspace{0.8cm} IIT Bombay\\
  \hspace{0.8cm}\href{mailto:satushparikh@gmail.com}{\textcolor{black}{\texttt{satushparikh@gmail.com}}} \\
  \And
 Soumya Basu\thanks{These authors contributed equally to this work and share senior authorship.} \\
  Google\\
  New York\\
  \href{mailto:basusoumya@google.com}{\textcolor{black}{\texttt{basusoumya@google.com}}} \\
  \And
 Avishek Ghosh$^{\ast}$ \\
 Department of Computer Science and Engineering\\
  IIT Bombay\\
  \href{mailto:avishek_ghosh@iitb.ac.in}{\textcolor{black}{\texttt{avishek\_ghosh@iitb.ac.in}}} \\
    \And
 \hspace{-0.7cm}Abishek Sankararaman\thanks{Work done outside AWS. The opinions in this draft are that of the author and does not necessarily reflect that of the employer.} \\
 \hspace{-0.7cm} Amazon Web Services\\
 \hspace{-0.7cm}Santa Clara\\
  \hspace{-0.7cm}\href{mailto:abishek.90@gmail.com}{\textcolor{black}{\texttt{abishek.90@gmail.com}}} \\
  }

\begin{document}
\maketitle
\begin{abstract}
  Sequential learning in a multi-agent resource constrained matching market has received significant interest in the past few years. We study decentralized learning in two-sided matching markets where the demand side (aka players or agents) competes for the supply side (aka arms) with potentially time-varying preferences to obtain a stable match. Motivated by the linear contextual bandit framework, we assume that for each agent, an arm-mean may be represented by a linear function of a known feature vector and an unknown (agent-specific) parameter. Moreover, the preferences over arms depend on a latent environment in each round, where the latent environment varies across rounds in a non-stationary manner. We propose learning algorithms to identify the latent environment and obtain stable matchings simultaneously. Our proposed algorithms achieve instance-dependent logarithmic regret, scaling independently of the number of arms, and hence applicable for a \emph{large} market.
\end{abstract}


\section{Introduction}
Matching markets have become increasingly relevant in various modern applications, including school admissions, organ transplantation, and job matching. Traditionally, these markets were studied assuming that demand-side agents (also referred to as players or agents) and supply-side agents (also referred to as arms) have fixed, known preference rankings over one another. The primary goal in such settings is to achieve a \textit{stable matching}, where no agent-arm pair has the incentive to deviate from their assigned match. Stability in this context corresponds to a Nash equilibrium. The seminal work of Gale and Shapley~\cite{gale1962college} introduced an algorithm where agents, through repeated proposals and deferred acceptance, converge to a stable matching.

However, in many real-world applications, such as crowdsourcing, online labor markets, and gig economy platforms, agents do not know their preferences \textit{a priori}. Instead, they must learn these preferences through repeated interactions with the environment. In such cases, it is typically assumed that arms already know their preferences over agents, which leads to a setting known as \textit{one-sided learning}. Recent research has focused extensively on this setup. Studies such as~\cite{liu_centralized_2020, liu_decentralized_2021, ucbd3, basu_2021, kong2023playeroptimalstableregretbandit} have developed learning algorithms that integrate Gale-Shapley dynamics with bandit-feedback-based techniques. 

This paper focuses on a matching market with $N$ agents and  $K$ arms. Markets of this nature are commonly observed in real-world platforms such as Amazon Mechanical Turk, TaskRabbit, UpWork, and Jobble. In these platforms, workers are modeled as agents, and different types of work (e.g., pairwise inference, crowdsourcing, rating movies, or restaurants in Amazon Mechanical Turk) are modeled as arms. In these examples, typically external factors like the geographical location, work experience, work timings (weekday vs weekend), time of the week or month, seasonal activities (holidays), company demands, and personal information influence the agent's preference. One canonical way to model these external factors is through a feature (or context) vector.
To address this, we adopt the framework of linear contextual bandits \cite{chu11a, OFUL}, where rewards are modeled as a linear function. We assume a linear structure because of its simplicity and theoretical tractability. Specifically, the expected reward for an agent matching with an arm is expressed as the inner product of an agent-specific latent reward vector and an agent-arm specific feature vector, which is revealed to the agent. Agents must learn their latent vectors to rank arms effectively based on their feature vectors, thereby eliminating the need for extensive exploration of all arms.

In this work, we allow the feature vectors to vary with time, making it a non-stationary learning problem.
In standard linear bandits~\cite{OFUL}, changes in feature vectors do not pose significant challenges, as an agent can play the optimal arm once its latent vector is learned. However, for multiple agents, each must play their stable matched arms. If preferences change due to arbitrarily varying feature vectors, our decentralized system would require $N^2$ rounds to relearn a stable matching using the Gale-Shapley algorithm. Therefore, allowing feature vectors to change arbitrarily is not viable in our setup, as it may lead to linear regret. 


To address this, we introduce additional structure to the problem by defining the notion of an environment, as detailed in Subsection~\ref{ssec:env-regret}. An environment is (uniquely) characterized by a preference ranking of agents (as well as arms). Moreover, these preferences change across environments. We assume a finite number of latent environments exist, and at each time step, one of these environments is active. Given that agents have learned their latent vectors, the rankings of arms they form based on mean rewards, observed after processing the feature vectors, would uniquely correspond to one of the latent environments. While small feature perturbations can occur within a particular environment, feature vectors from different environments would lead to distinct rankings.

\emph{Motivating example:} Let us look at the application of Amazon Mechanical Turk. Based on external factors, the (monetary) reward of different tasks may vary with time. For example, before a major product release, companies may pay more for tasks like data cleaning, labeling, processing, rating search result accuracies, identifying issues, and removing redundancies. On the other hand, jobs like filling out surveys or answering questions on various topics have a (more or less) steady pay regardless. Since the worker's goal is to earn maximum reward,  it is in the worker's best interest to identify the temporal variability and change their preference accordingly.  


This contextual framework enables richer and more realistic modeling of markets. It is partly inspired by latent bandit models~\cite{hong_latent_bandits,gopalan2016low}, which combine offline learning of complex models from historical interactions with online decision-making, where agents select actions after identifying the latent state. Apart from bandits literature, latent variable estimation is a classical problem in statistics with applications in clustering and mixture models.

\subsection{Our Contribution}
\emph{Contextual Matching Market:} We address a decentralized setup where agents make independent decisions without a central coordinator. A central challenge in such multi-agent systems is resolving collisions, which occur when multiple agents attempt to match with the same arm. Each arm maintains a preference ranking over agents, and in the event of a collision, the arm selects its most preferred agent, granting them a stochastic reward. The remaining agents are informed of their rejection and receive zero reward. This decentralized structure prevents an agent from independently running a standard linear bandit algorithm, such as OFUL \cite{OFUL}, as the regret incurred due to collisions is not characterizable.

\emph{Algorithm:} We propose a novel algorithm, Environment Triggered Phased Gale Shapley (\alg \!)  which employs round robin exploration, eliminating collisions.\ \alg  leverages least squares estimates to approximate the latent preference vector $\thetai$, which governs the ranking process. Using these estimates, we derive concentration bounds that quantify the deviation between the estimated mean rewards and their actual values. As exploration progresses, the confidence intervals around the estimated rewards narrow, leading to increasingly precise estimates. Once the confidence intervals for the top $N$ arms no longer overlap, agents can confidently identify the correct ranking for different environments. We detail our proposed algorithm in Section~\ref{sec:Contextual Matching Market Algorithm}. Additionally, we present an improved algorithm, Improved Environment Triggered Phased 
Gale Shapley (\improved\!) in Section~\ref{sec:Multienv Contextual Matching Market Algorithm}, which incorporates partial rank matching through  Kendall tau distance and refined reward gap conditions to improve regret.

\emph{Theoretical Guarantees:} Both \alg and \improved achieve logarithmic, instance-dependent regret per agent, aligning with the order-wise lower bounds established in \cite{ucbd3} concerning the time horizon and reward gap. With $d$ dimensional features and \emph{appropriately} defined minimum gap $\Delta_{\min}$, \alg suffers a regret of $\mathcal{O}(\frac{d^2 \log T}{\Delta_{\min}^2} + E N^2) $, where $E$ is the number of environments. It is important to note that the regret is independent of the number of arms $K$ (matching the results of classical contextual linear bandits \cite{chu11a}). Moreover, the dominating term in the regret has no dependence on the number of environments $E$. Hence, our framework can tolerate \emph{large markets} where the number of arms $K$ can be reasonably big. Furthermore, we can also tolerate many environments (which is practical). Moreover, \improved improves the regret by removing the dependence on $\Delta_{\min}$ and replacing it with a different instance gap, which might be much bigger than $\Delta_{\min}$.

Notably, our algorithm does not require prior knowledge of the gaps or the horizon length.



\subsection{Related Work}
\vspace{-1mm}
Multi-armed bandits \cite{thompson1933,lai_1985,Lattimore} is a popular framework for capturing learning under uncertainty in single and multi-agent settings \cite{marl-book}. For multi-agent learning, game-theoretic approaches are used \cite{anima,avner2014concurrentbanditscognitiveradio,salgia}.

\emph{Contextual Bandits:} When the action set of the agents is large, typically structural assumptions on arm means are assumed, leading to the popular problem of contextual bandits. Using linear assumption, \cite{chu11a,OFUL,varsha_dani_kakade_2008} show sub-linear regret in this setup. Extension to non-linear framework is also studied in \cite{simchi2022bypassing,foster2020beyond}. However, all these works focus on learning for a single agent.

\emph{Markets and Bandits:} \cite{liu_centralized_2020} formalized the problem of markets in a learning framework involving multiple agents and arms in a centralized setting, where a central arbiter was available. Subsequent works extended this foundation, addressing the centralized nature, global information structure, and regret bounds. 
Centralized systems, however, have been criticized for potential data breaches and privacy concerns.
Decentralized setups were explored in \cite{liu_decentralized_2021,ucbd3,basu_2021,maheshwari_2022}, which introduced trade-offs between the amount of globally available information and structural assumptions on preferences, such as Serial Dictatorship, uniqueness consistency, and regret bounds. \textcolor{black}{Markets have been studied from the perspective of Reinforcement Learning \cite{min2022learnmatchregretreinforcement_reviewer3}, within the framework of transferable utilities \cite{jagadeesan2023_reviewer1}, and under settings with complementary preferences and quota constraints \cite{li2024twosidedcompetingmatchingrecommendation_reviewr2}.}

\cite{kong2023playeroptimalstableregretbandit} proposed an explore then commit type framework for this setting, which performs as well as the UCB type algorithm. Efforts have also been made to explore two-sided learning, where neither side initially knows the preferences of the other \cite{tejas_2024}. 
Matching markets have been investigated under game-theoretic approaches as well \cite{srikanth_uiuc}.

However, none of these algorithms address the linearly structured dynamic market problem, which we do by introducing a contextual matching market. This is also useful for \emph{large} markets when the number of items is large. Interested readers can refer to a recent survey on multiagent bandits \cite{survey}



%

\subsection{Preliminaries and Notation}
\label{ssec:prelim}
The inner product between two vectors \( \boldsymbol{u} \) and \( \boldsymbol{v} \) is denoted by \( \langle \boldsymbol{u}, \boldsymbol{v} \rangle \). For an integer \( M \), the notation \([M]\) represents the set \(\{1, 2, \ldots, M\}\). The minimum and maximum eigenvalue of square matrix $ A $ are denoted as $ \lambda_{\min}(A) $ and $ \lambda_{\max}(A) $.  Given a positive semi-definite matrix \( A \), the weighted norm of a vector \( \boldsymbol{x} \) with respect to \( A \) is defined as $\| \boldsymbol{x} \|_A = \sqrt{\boldsymbol{x}^T A \boldsymbol{x}}.$
A \textit{partial ranking} is a ranking where not all pairs of elements are necessarily ordered, allowing for ties or unresolved comparisons. Given two partial rankings \( pr \) and \( pr' \), we define the set of \textit{inverted pairs} as:  

\begin{align*}
    {\rm Inv}(pr, pr') = \{(k,j): \big((k \underset{pr}{\succ} j) \wedge (k \underset{pr'}{\prec} j)\big) \vee \big((k \underset{pr}{\prec} j) \wedge (k \underset{pr'}{\succ} j)\big) \}.
\end{align*}
The notation $\wedge (\vee)$ stands for the logical AND (OR) operation.
The cardinality of the inversion set is commonly referred to as the Kendall Tau distance \cite{kendall} as  $KT(pr, pr') = |{\rm Inv}(pr, pr')|.$ Notations have been summarized in Table \ref{tab:notations}.

\section{Problem Setup}
\label{subsec:problem_setup}
\subsection{Contextual Matching Market}

We consider a two-sided matching market involving $N$ agents and $K$ arms, where $K \geq N$. At the start of each round $t = 1, \dots, T$, an environment is active. The formal definition of an environment is provided in {Subsection \ref{ssec:env-regret}}. Each agent proposes one arm per round. Each arm is assumed to have a fixed preference ranking over the agents based on the active environment; thus, if an arm receives multiple proposals, it accepts the one from its most preferred agent. Agents that are not matched receive a rejection notification and obtain zero reward, while matched agents receive a stochastic reward.

When agent $i$ is matched with arm $j$, the agent receives a reward given by $ r_{i}(t) = \mu_{i, j}(t) + \eta_{i}(t)$
where $\mu_{i, j}(t)$ is the mean reward, and $\eta_{i}(t)$ is an independent sample from a zero-mean, 1-subgaussian distribution. Let $m_i(t)$ denote the arm matched with agent $i$ at time $t$, with $m_i(t) = \phi$ indicating that the agent is unmatched. 

Agents have access to a common \textit{information board}, where they can both write and read during a round, similar to the settings in \cite{avishek_nonstationary,kong2023playeroptimalstableregretbandit}.
\footnote{It is possible to eliminate the need for a shared information board at the cost of implementing more complex communication protocols while maintaining the same orderwise regret.}

We assume a linear contextual reward model. Each agent $i$ has an associated latent parameter $\boldsymbol{\theta}_i \in \mathbb{R}^d$. In each round, agent $i$ observes a feature vector $\boldsymbol{x}_{ij}(t) \in \mathbb{R}^d$ for each arm $j \in [K]$.
\footnote{For notational convenience, we assume that the feature vectors for all agents have the same dimension $d$. The model can be generalized to allow different dimensions for different agents.} The mean reward when agent $i$ is matched to arm $j$ at time $t$ is given by $ \mu_{i, j}(t) = \langle \boldsymbol{x}_{ij}(t), \boldsymbol{\theta}_i \rangle.$
The latent parameter $\boldsymbol{\theta}_i$ remains fixed across all rounds and must be learned to help the agent establish preferences over arms. We assume that for each agent $i$ and each round $t$, all mean rewards $\mu_{i, j}(t)$ are \textit{distinct}. Agent $i$ prefers arm $j$ over arm $j'$ if $\mu_{i,j}(t) > \mu_{i,j'}(t)$. A matching is considered \textit{stable} if no agent-arm pair prefers each other over their current matches. The \textit{agent-optimal} stable matching is where each agent is matched to its most preferred arm among all possible stable matches. The Gale-Shapley algorithm always produces a unique agent-optimal matching.

\subsection{Environment and Regret}  
\label{ssec:env-regret}  

As motivated in the Introduction, we aim to capture variations in agent-arm interactions induced by latent contextual factors such as the type of work, worker demand, and time of day. These latent factors influence the rewards of agent-arm pairs by determining the arrival of feature vectors $\mathbf{x}_{i,j}(t)$, which in turn shape the preferences of both agents and arms. While each agent $i$ has a fixed latent parameter $\boldsymbol{\theta}_i$ across all environments, the environment itself dictates how feature vectors are realized and, consequently, how preferences are formed.  

Each environment specifies the ranking of all arms for every agent and the ranking of all agents for every arm. When an environment is active in a given round, arms are aware of their preferences over agents, but agents do not directly observe which environment is in effect. Instead, they infer their preferences from the observed feature vectors of the arms. Within a fixed environment, the ranking induced by the mean rewards, computed based on an agent's latent vector and the arms' feature vectors, remains consistent. However, since agents do not initially know their latent preference vectors, they must learn them over time to establish their rankings over arms.

The feature vectors $\mathbf{x}_{i,j}(t)$ are indexed by time $t$, allowing for potential non-stationarity. This accounts for small perturbations in the feature vectors within a given environment and variations in the feature vectors across different environments.

Having defined the notion of an environment, we now consider the case of multiple environments. The environment at any given time $t$ is drawn from a finite set $\mathcal{E}$ of size $E$. There is no specific structure governing the sequence of environmental occurrences. The environment at time $t$, denoted $e(t)$, is unknown to the agents but is the same for all agents and arms in that round. Each agent knows the total number of agents $N$ and the number of environments $E$.  

Denote by $\rho_i^e \in \mathbb{R}^K$ the preference vector of agent $i$ over arms in environment $e$. Let $\rho_i^e[1:N]$ represent the top $N$ preferences. We impose the following assumption for two distinct environments $e, e' \in \mathcal{E}$.

\begin{assumption}[Environment Structure]\label{a:env}
  The system satisfies the following condition: for any two distinct environments $e, e' \in \mathcal{E}$, the top $N$ preferences of each agent $i$ must be distinct, i.e., the ranking vectors are not identical, $\rho_{i}^{e}[1:N] \neq \rho_{i}^{e'}[1:N]$.
\end{assumption}
 
 The above assumption ensures that the environments are distinguishable through preference rankings, a property that we justify later. The goal is to achieve a stable matching for each environment while minimizing cumulative expected regret across all environments.  
  
The expected cumulative regret for agent $i$ is defined as  
\begin{equation}  
\mathbb{E}[R_T^{(i)}] = \mathbb{E}\left[\sum_{t=1}^{T} \left( \mu_{i, m_i^{*}(e(t))}(t) - \mu_{i,m_i(t)}(t) \right)\right],  
\end{equation}  
where $e(t)$ is the environment active at time $t$, $m_i^{*}(e(t))$ denotes the arm assigned to agent $i$ in the agent-optimal stable matching of environment $e(t)$, and $m_i(t)$ represents the arm matched with agent $i$ at time $t$.  
\paragraph{Necessity of Distinct Rankings in Environments:}  
\label{justify_env}  
We assume that each agent's ranking of the top $N$ arms differs for two distinct environments. Although this may seem restrictive, we argue below that without this assumption, when the environment $e(t)$ is latent, sub-linear regret cannot be achieved.\footnote{One may argue that the information board can be used to pass information. However, as discussed earlier, this is only meant to simplify exposition and not to alter the decentralized nature of our system. A system where its potential is fully utilized is beyond the scope of this work.}  

Consider a simple setting with two agents and two arms in a serial dictatorship setting where both arms prefer Agent 1 over Agent 2. If an agent maintains the same ranking across two environments, it can impede the convergence of the Gale-Shapley matching process. The agents' preferences in two environments are shown in Table \ref{table: distinct_rankings_example}. 

Suppose agents have learned the arms rankings in both environments and now seek a stable match for each. Consider a scenario where the two environments alternate at successive time steps. At time \( t_1 \), environment \( e_1 \) appears, and agent \( p_1 \) is matched with arm \( a_1 \), while agent \( p_2 \) is rejected. At time \( t_1 + 1 \), environment \( e_2 \) appears. Recognizing the new environment, \( p_1 \) proposes to arm \( a_2 \), while \( p_2 \) also proposes to \( a_2 \) but is rejected again. As a result, \( p_2 \) remains unmatched, and this pattern repeats indefinitely.

\begin{table}[h]  
  \centering  
  \begin{tabular}{c|c|c}  
      & $ e_1 $ & $ e_2 $ \\ \hline  
      $ p_1 $ & $ a_1 \succ a_2 $ & $ a_2 \succ a_1 $ \\ \hline  
      $ p_2 $ & $ a_1 \succ a_2 $ & $ a_1 \succ a_2 $ \\  
  \end{tabular}  
  \caption{Agent preferences in two environments}  
  \label{table: distinct_rankings_example}
\end{table}

\section{Contextual Matching Market Algorithm}\label{sec:Contextual Matching Market Algorithm}
This section introduces our decentralized algorithm and establishes its corresponding regret bound. Following the approach of \cite{kong2023playeroptimalstableregretbandit}, we adopt a UCB/LCB-based Explore-Then-Commit (ETC) framework for a single environment. We propose the \emph{Environment-Triggered Phased Explore and Gale-Shapley} (\alg \!) algorithm. Throughout this section, we assume the existence of a uniform minimum reward gap, $\Delta_{\min}$ (formally defined later), across all environments.

If environment detection fails and no exploration phase is already in progress, the agent initiates a new exploration phase of exponentially increasing length, $Block(l)$. A shared flag, $\mathcal{B}$, facilitates this triggering. During an ongoing exploration phase, agents perform round-robin exploration. If environment detection succeeds, the agent either retrieves the previous environment state, $(s_i, \sigma_i) = D[e]$, or initializes a new state, $D[e] = (1, \sigma_i)$ and If no exploration phase is in progress, the agent plays the corresponding Gale-Shapley index. The index $s$ tracks the current proposal number in the Gale-Shapley algorithm, incrementing if the agent remains unmatched when a step is played. The ranking of the top $N$ preferences is stored in memory as $\sigma$. This iterative process adapts to environmental shifts while leveraging experience to recover stable matches efficiently. Agents may discover rankings across environments in different orders, so the key 
$e$ may not correspond to the same environment for all agents. However, as each environment has a unique top-$N$ ranking for each agent due to Assumption~\ref{a:env}, in line $11$, the correct top-$N$ rank is retrieved with high probability. In  particular, in round $t$ for all agent $i$ that enters line $11$  the retrieved $\sigma_i[e_i(t)] = \rho_i^{e(t)}[1:N]$ (up to top $N$ ranks)  for the true environment $e(t)$, even though the indices $e_i(t)$s can be different.

Let the least squares estimate of $\thetai$ at time t be $\et$. Define $\v$ as the matrix formed by the outer product of the feature vectors seen during exploration up to time $t$. Let the set $\{\boldsymbol{v}_{\ell}: \ell \in [d]\}$   form a set of  {standard orthonormal basis} vectors for $\mathbb{R}^d$. Define the following:
 \begin{align}
  \v & = \sum_{s=1}^{t}
  \boldsymbol{x}_{i,m_i(s)}(s) \boldsymbol{x}_{i,m_i(s)}^T(s) \notag \\
  \et & = \left( \v \right)^{-1}( \sum_{s=1}^{{t}}
  r_i(s) \boldsymbol{x}_{i,m_i(s)}(s) ) 
  \label{eq:theta_estimate} 
\end{align}
Denote by $\hat{\mu}_{ij}(t) = \langle \et, \boldsymbol{x}_{i,j}(t) \rangle$ the estimated mean of arm $j$ at time $t$ as seen by agent $i$. Agent $i$ constructs the $\ucb$ and $\lcb$ for the arm $j$ at time $t$ as follows: 
\begin{equation}
\begin{aligned}
\ucb_{ij}(t) &= \langle \et, \boldsymbol{x}_{i,j}(t) \rangle + w_{i}(t,\boldsymbol{x}_{i,j}(t)), \\
\lcb_{ij}(t) &= \langle \et, \boldsymbol{x}_{i,j}(t) \rangle - w_{i}(t,\boldsymbol{x}_{i,j}(t)),
\end{aligned}
\label{eq:ucb_lcb_equations}
\end{equation}
where 
\begin{align}
  \label{w_ijt}
    w_{i}(t,\boldsymbol{x}) & = \sum_{\ell\in [d]} |\langle \boldsymbol{x}, \boldsymbol{v}_{\ell}\rangle| \sqrt{2 \|\boldsymbol{v}_{\ell}\|_{\vi}^2 \log \left( t^2   \right)}.
\end{align}

\emph{Identifying Top-$N$ Arms:} In Appendix \ref{max_steps_in_gale_shapley}, we show that for $N$ agents, each agent’s stable matched arm lies among the top $N$ arms in the ranking of $K$ arms. Thus, it suffices for an agent to identify the top $N$ arms rather than rank all $K$ arms. The condition for correctly identifying the top $N$ arms at time $t$ is:
\begin{align}
\forall a \in [N],\ \lcb_{i\sigma(a)}(t) > \!\!\!\! \max_{\sigma{(a+1)} \leq c \leq \sigma{(K)}}\!\!\!\! \ucb_{i c}(t),
\label{eq:separation}
\end{align}
where $\sigma: [K] \to [K]$ is a ranking permutation.






\begin{algorithm}[H]
\caption{ETPGS: Environment-Triggered Phased Gale-Shapley (at agent $i$)}
\label{alg:alg1}
\begin{algorithmic}[1]

\State \textbf{Explore phase length:} $Block(l) := 2^l$
\State \textbf{Blackboard variable:} $\mathcal{B}$
\State \textbf{Initialize:} $D[e{:}(s_i, \sigma_i)] = \{\}$, $\tau_{\text{end}} = 0$, $l = 0$

\vspace{0.5em}
\For{$t \geq 1$}
    \State \textbf{Reset:} $\mathcal{B} \gets 1$ \Comment{Environment Recovery}

    \State $\sigma_i(t) \gets 
        \begin{cases}
            \sigma[1{:}N], & \text{if } \sigma \text{ satisfies Eq.~\eqref{eq:separation}} \\
            \emptyset, & \text{otherwise}
        \end{cases}$

    \If{$\sigma_i(t) \neq \emptyset$}
        \If{$\sigma_i(t) \notin \{\sigma_i[e] \mid e \in D\}$}
            \State Add new environment: $D[e] \gets (1, \sigma_i(t))$
        \EndIf
        \State $e_i(t) \gets e$ such that $\sigma_i[e] = \sigma_i(t)$
    \Else
        \State $\mathcal{B} \gets \mathcal{B} \land 0$
    \EndIf

    \vspace{0.5em}
    \Comment{Exploration Triggering}
    \If{$\mathcal{B} = 0$ and $t > \tau_{\text{end}}$}
        \State Restart explore: $\tau_{\text{end}} \gets t + Block(l)$, $l \gets l + 1$
    \EndIf

    \vspace{0.5em}
    \Comment{Explore or Gale-Shapley Matching}
    \If{$t \leq \tau_{\text{end}}$} \Comment{Exploration Phase}
        \State Play arm: $m_i(t) \gets ((i + t) \bmod K) + 1$
        \State Observe reward $r_i(t)$
        \State Update estimates: $\hat{\theta}(t), \mathrm{UCB}(t), \mathrm{LCB}(t)$
    \Else \Comment{Exploitation Phase}
        \State Retrieve $(s, \sigma) = D[e_i(t)]$
        \State Play GS index: $m_i(t) \gets \sigma[s]$
        \If{$m_i(t) = \phi$}
            \State Update environment pointer: $D[e_i(t)] \gets (s + 1, \sigma)$
        \EndIf
    \EndIf
\EndFor

\end{algorithmic}
\end{algorithm}

\subsection{Regret Analysis}
\label{subsec:regret_ana}
We now present the regret guarantees for the multiple-environment problem. We make the following assumption to ensure we have diverse feature vectors that can be used for learning $\thetai$ for each agent $i$.
To guarantee the shrinkage of confidence intervals over time, we impose a structural condition requiring the minimum eigenvalue of the outer product matrix of feature vectors to grow with the number of exploration samples. We will discuss the implications of this assumption in detail. 
 
\begin{assumption}[Full Rank Feature Vectors]
  \label{a:min-eigen-value}
  For each agent  $i \in [N]$, the set of arms  $[K]$ can be divided into non-overlapping groups of $d$ distinct arms each  $\mathcal{G} = \{G_1, G_2, \dots, G_{\lfloor K / d \rfloor}\}$, such that any group $G \in \mathcal{G}$  and any $\{t_1, \dots, t_d \} \subseteq [T]^d$ the following is true
  $$
  \lambda_{\min}(\sum_{j = 1}^d \boldsymbol{x}_{i,G(j)}(t_j) \boldsymbol{x}_{i,G(j)}(t_j)^T) \geq \kappa. 
  $$
\end{assumption}
Note that such an assumption on the spectrum of the data matrix has appeared before in the literature. This is especially needed if one needs parametric inference (a bound on how $\theta_i$ behaves) on top of regret minimization. In  \cite{gentile2014online}, the authors use this for clustering in linear bandits. Furthermore, \cite{chatterji2020osom} uses a similar assumption for choosing \emph{the best of both worlds} between linear bandits and unstructured bandits. Moreover, in the context of model selection (adaptation) and personalization, similar assumptions appear in \cite{ghosh2021problem,ghosh2022multi,ghosh2022breaking}. Specific scaling of $\kappa$ is discussed in \cite{debanshu2023min}. However, for our purposes, we just require $\kappa >0$, not any particular scaling.

We want to mention further that with additional assumptions, such as heteroskedastic noise (arm-dependent noise, where the variance of the noise goes to zero when the action taken is close to the optimal action) \cite{lumbreras2024linear} could avoid such a spectral assumption. However, the analysis in \cite{lumbreras2024linear} is quite non-trivial, and the obtained regret is worse (polylogarithmic instead of logarithmic). In this paper, for simplicity, we assume the above spectral assumption and want to get a characterization of competition among multiple agents in a linear bandits framework. 


\emph{Minimum Gap:} Recall that it suffices for the agents to learn ranking correctly for their top $N$ preferences, and to do this, it is sufficient to rank top $N+1$ arms correctly. Let ${\rm Top}_i(N+1) \subseteq [K]$ denote the top $N+1$ arms for the agent $i$ in a given ranking. The minimum reward gap at time $t$, for agent $i$ is denoted as 
\begin{gather}
  \Delta_{\min, i}(t)\!\! = \!\!\!\!\!\min_{j,j'\in {\rm Top}_i(N+1), j \neq j'}  |\langle \boldsymbol{x}_{i,j}(t)\!\! -\!\! \boldsymbol{x}_{i, j'}(t), \boldsymbol{\theta}_i\rangle|, \label{eq:min-gap-topN}
\end{gather}
and $\Delta_{\min} = \min_{i,t} \Delta_{\min, i}(t)$ denotes the {\em uniform minimum reward gap}. We also define the maximum reward for agent $i$ as $\mu_{i, \max} = \max_{t, e \in \mathcal{E}} \mu_{i, m^*_i(e)}(t)$, and $L$ as the maximum norm of any feature vector. When the confidence widths become sufficiently small to resolve the gap of \( \Delta_{\min} \), the agent can determine all rankings and correctly identify the environments. We show in Appendix \ref{lem:confidence_width} that after $\mathcal{T}(t)$ exploration rounds by time $t$, the confidence width around the estimated mean which contains the true mean with high probability is upper bounded as $\frac{dL}{\sqrt{\kappa}}\sqrt{\frac{\log t}{\mathcal{T}(t)}}$. Hence to resolve a gap of $\Delta$ agents would require $\mathcal{O}\left(\tfrac{d^2L^2 \log(T)}{\kappa\Delta^2}\right)$ rounds of exploration. Here, $L$ is an upper bound on the norm of feature vectors. With this, we are ready to present the main result of this section.

 \begin{theorem}
  \label{thm:regret_bound}
  Under Assumptions~\ref{a:env} and ~\ref{a:min-eigen-value}, for systems with uniform reward gap $\Delta_{\min} > 0$ the regret of agent $i$ when all agents run \alg \ is 
  \begin{align*}
      &\mathbb{E}[R_T^{(i)}]  \leq  \left( \tfrac{64d^2L^2 \log(T)}{\kappa \Delta_{\min}^2}  + EN^2+\tfrac{Nd\pi^2}{3} \right)\mu_{i,\max}
  \end{align*}
\end{theorem}

The first term on the right-hand side represents the regret incurred during the net exploration until all agents have successfully identified the rankings. The convergence of the Gale-Shapley algorithm requires at most $N^2$ steps, and for $E$ environments, this results in a total of $EN^2$ steps. The final term arises as an artifact of the regret analysis under a specific constructed bad event.

\paragraph{Remark on Regret Scaling:} Note that our regret scaling is logarithmic in $T$.  The dependence on \emph{gap} is also inverse squared. The two match the lower bound presented in \cite{ucbd3}. More importantly, our regret is independent of the number of arms $K$ and has a $d^2$ dimension dependence.  Furthermore, the dominant regret term on the right-hand side is independent of the number of environments $E$. For the special case of $E=1$, our \alg \ reduces to the single-environment setting and recovers orderwise the regret bound presented in \cite{kong2023playeroptimalstableregretbandit}. The proof of Theorem \ref{thm:regret_bound} is provided in Appendix \ref{proof: delta_min}

\section{Improved Algorithm with Partial Rank Matching and Reward Gap Period}
\label{sec:Multienv Contextual Matching Market Algorithm}
Prolonged exploration to shrink confidence widths for distinguishing the smallest gap, $\Delta_{\min}$, can be expensive, as $\Delta_{\min}$ may approach zero. This section investigates how leveraging the problem's inherent structure can significantly reduce exploration time. Specifically, we exploit the existence of finite reward gap periods instead of assuming the existence of a constant uniform minimum reward gap. The central idea is that if a gap of $\Delta (> \Delta_{\min})$ can be resolved, agents can recover rankings across the environments if this gap $\Delta$ has shown up at least once in each of the environments. Once rankings are established, the next challenge is identifying the current environment. Even if agents cannot fully resolve the rankings of the top arms, they can still attempt to infer the environment using partial rankings. We enhance the Environment Recovery component of \alg, introducing our improved algorithm: \emph{Improved Environment-Triggered Phased Explore and Gale-Shapley} (IETP-GS).

 \improved provided in Algorithm~\ref{alg:alg2} makes two changes over \alg. First, alongside trying to recover the top-$N$ rank, it also recovers the partial $pr_i(t)$ at round $t$ using $\ucb_i(t)$ and $\lcb_i(t)$ as follows: 
Agent $i$ computes a partial rank 
\begin{align}\label{eq:partial}
pr_i(t) \triangleq 
\begin{cases}
 j' \underset{pr_i(t)}{\succ} j, & \text{if } \lcb_{ij'}(t) > \ucb_{ij}(t), \\ 
 j \underset{pr_i(t)}{\succ} j', & \text{if } \lcb_{ij}(t) > \ucb_{ij'}(t), \\ 
 j \underset{pr_i(t)}{=} j', & \text{otherwise}.
\end{cases}
\end{align}
\textcolor{black}{Let $pr_i(t)$ denote the ranking of arms obtained through partial rank restricted to top $N$ positions. Note that multiple arms may have the same rank in $pr_i(t)$.} Next, in addition to identifying the top-$N$ ranking, \improved attempts to find a unique environment $e \in D$ such that the partial rank $pr_i(t)$ has zero Kendall-Tau distance from the stored ranking $\sigma[e]$. If such a unique $e$ exists, it corresponds to the active environment. Once an environment is successfully identified, the \improved falls back to the Gale-Shapley algorithm; otherwise, it triggers further exploration, similar to \alg.

\begin{algorithm}[H]
\caption{IETPGS: Improved Environment-Triggered Phased Gale-Shapley (at agent $i$)}
\label{alg:alg2}
\begin{algorithmic}[1]

\State \textbf{Explore phase length:} $Block(l) := 2^l$
\State \textbf{Blackboard variable:} $\mathcal{B}$
\State \textbf{Initialize:} $D[e{:}(s_i, \sigma_i)] = \{\}$, $\tau_{\text{end}} = 0$, $l = 0$

\vspace{0.5em}
\For{$t \geq 1$}
    \State \textbf{Reset:} $\mathcal{B} \gets 1$ \Comment{Global reset}

    \vspace{0.5em}
    \Comment{Environment Recovery}
    \State $\sigma_i(t) \gets 
        \begin{cases}
            \sigma[1{:}N], & \text{if } \sigma \text{ satisfies Eq.~\eqref{eq:separation}} \\
            \emptyset, & \text{otherwise}
        \end{cases}$
    \State Recover partial rank: $pr_i(t)$ according to Eq.~\eqref{eq:partial}

    \If{$\sigma_i(t) \neq \emptyset$ \textbf{or} $\big(|D| = E$ \textbf{and} $\exists\, e \in D \text{ such that } \mathrm{KT}(pr_i(t), \sigma_i[e]) = 0 \big)$}
        \If{$\sigma_i(t) \neq \sigma_i[e]$ for all $e \in D$ \textbf{and} $\sigma_i(t) \neq \emptyset$}
            \State Store new environment: $D[e] \gets (1, \sigma_i(t))$
        \EndIf
        \State Set $e(t) \gets e$ such that $\sigma_i[e] = \sigma_i(t)$ \textbf{or} $\mathrm{KT}(pr_i(t), \sigma_i[e]) = 0$
    \Else
        \State $\mathcal{B} \gets \mathcal{B} \land 0$
    \EndIf

    \vspace{0.5em}
    \Comment{Exploration Triggering}
    \State Same as Algorithm~\ref{alg:alg1} (ETPGS)

    \vspace{0.5em}
    \Comment{Explore or Gale-Shapley Matching}
    \State Same as Algorithm~\ref{alg:alg1} (ETPGS)

\EndFor

\end{algorithmic}
\end{algorithm}

\subsection{Improved Regret with \improved}
\subsubsection{Reward Gap Period} We analyze the regret without the strong uniform $\Delta_{\min}$ assumption. Instead, we assume a relaxed reward gap assumption that allows for arbitrary small $\Delta_{\min}$. Let $\nu_e$ denote the $\nu_e$-th appearance of environment $e$, and let $t_e(\nu_e)$ denote the time of such occurrence. 
The reward gap period for environment \( e \), denoted by $P_e(\Delta)$, is defined as the number of occurrences required to observe at least one minimum gap $\Delta$ in the arrival sequence of environment $e$ at any point. 

\begin{definition}[Reward gap period]\label{def:modified-gap}
The reward gap function $P_e(\Delta)$ for an environment $e\in [E]$ maps a gap $\Delta > 0$ to the maximum number of occurrences of $e$ needed to observe a gap of at least $\Delta$ across all agents, i.e. 
\begin{align*}
P_e(\Delta) \triangleq \max_{i \in [N]} \max_{\nu_e\geq 1} &\left(\min \{ \delta\nu_e: \delta\nu_e \geq 0, \right.\left. \Delta_{\min, i}(t_e(\nu_e + \delta\nu_e)) \geq \Delta\}\right).
\end{align*}
\end{definition}
Note that for a system with $P_e(\Delta)$ reward gap, the following sequence where  $\Delta_{\min, i}(t) = 0$  for at most $(P_e(\Delta) - 1)$ consecutive appearances of an environment $e$
$$
(\dots, \Delta,\underbrace{0,\dots, 0 \dots, 0}_{(P_e(\Delta) - 1) \rm{times} }, \Delta, \dots).
$$
Also, note that the $P_e(\Delta)$ increases with increasing $\Delta$, i.e., finding large gaps requires more time.

 
\subsubsection{Rank Discovery}
We evaluate the time required for agents to resolve rankings across environments. For an environment $e \in \mathcal{E}$, let $\nu_{i,e}$ denote the earliest number of occurrence by which agent $i$ resolves the rankings for environment $e$:
\[
\nu_{i,e} = \min \left\{\nu : \Delta_{\min, i}(t_e(\nu)) \geq \frac{8dL}{\sqrt{\kappa}} \sqrt{\frac{\log(t_e(\nu))}{\mathcal{T}(t_e(\nu))}} \right\},
\]
where $\Delta_{\min, i}(t)$ is defined in Eq.~\eqref{eq:min-gap-topN}, and $\mathcal{T}(t)$ is the cumulative number of exploration rounds up to time $t$ across all environments, and $t_e(\nu)$ is the round of $\nu$-th arrival time for environment $e$.

The number of exploration rounds required to resolve a gap $\Delta$ satisfies:
$
\min\{n : \Delta \geq \tfrac{8dL}{\sqrt{\kappa}} \sqrt{\tfrac{\log(T)}{n}}  \} \leq \tau(\Delta),
$
where $\tau(\Delta) \triangleq \frac{64d^2L^2 \log(T)}{\kappa \Delta^2}$.
Also, let $\nu'_{e}(\Delta)$ denote the minimum number of occurrence of environment $e$ such that $\mathcal{T}(t(\nu'_{e}(\Delta))) \geq \tau(\Delta)$. In other words, for environment $e$, the agents can resolve a minimum rank of at least $\Delta$ starting from the $\nu'_{e}(\Delta)$-th occurrence.

We show that once $\mathcal{T}(t) \geq \tau(\Delta)$ there can be at most $P_e(\Delta)$ more occurrences of environment $e$, for any $e$ and any agent $i$, such that the rank of $e$ is not recovered. More precisely, we show that 
$$
\nu_{i,e} \leq \min_{\Delta > 0} \nu'_{e}(\Delta) + P_e(\Delta) .
$$
Based on this, we argue (in the appendix) that the total number of rounds the algorithm does not play the optimal stable matching due to the rank recovery process being at most 
$$\tau_{\max} \triangleq(\min_{\Delta > 0}\tau(\Delta) + \sum_{e\in [E]} P_e(\Delta)) .$$
\vspace{-2em}

\subsubsection{Rank Matching}
We now analyze the amount of extra exploration that may be required to ensure that there is no error in rank matching to the active environment. Results in this section hold with high probability. When all agents can identify the active environment correctly, they play a step of the Gale Shapley algorithm if the stable match arm is not discovered yet or play the stable match arm of that environment if it is found. If an agent is not able to identify the active environment, it triggers a further exponential exploration subphase. We next characterize $\Delta_{\rm min rank}(>\Delta_{\min})$ such that once confidence width has become small enough to resolve $\Delta_{\rm min rank}$, agents would be able to identify the active environment correctly.

Recall that once all rankings have been determined, the rankings of arms stored by \improved \ are given by $\sigma_i[e]$ for all $e \in [E]$ where $\sigma_i$ is stored separately in each agent's memory $D$. Agent $i$ can correctly identify the active environment if there is a unique $e \in [E]$ such that $KT(\sigma_i[e],pr_i(t))=0$. We defer this proof to Appendix \ref{sec:characterizing_exploration}.




We define the minimum rank gap for agent $i$ and time $t$ as
\begin{align*}
    \Delta_{i, {\rm rank}}(t) = &\underset{e \neq e(t)}{\min}  \max \{ |{\mu}_{ij}(t) - {\mu}_{ij'}(t)| :(j,j') \in {\rm Inv}(\rho_{i}^e[1:N], \rho_i^{e(t)}[1:N])\}    
\end{align*}
By taking the worst-case bound over time $t$, and agents $i$ we define 
$
\Delta_{\rm min rank} = \min_{1 \leq t \leq T}\hspace{0.5em}\min_{i \in [N]}  \Delta_{i, {\rm rank}}(t).
$
Note that $\Delta_{\rm min rank} \geq \min_{1 \leq t \leq T}\hspace{0.5em}\min_{i \in [N]} \Delta_{i, \min}(t)$ as the former takes max over the pairs in ${\rm Inv}(\rho_{i}^e[1:N], \rho_i^{e(t)}[1:N])$, and not over all top-$N$ arms, thus providing larger gaps. 



Consider a simplified setting with two agents, three arms, and two environments to motivate this. At each time, one of the two environments is observed, and on every $C$-th occurrence of a given environment, the expected arm rewards switch to:
\[
e_1 : \quad \boldsymbol{\mu}_1 = (1, 2\delta, \delta), \quad \boldsymbol{\mu}_2 = (2\delta, 1, \delta), \qquad e_2 : \quad \boldsymbol{\mu}_1 = (1, \delta, 2\delta), \quad \boldsymbol{\mu}_2 = (\delta, 1, 2\delta).
\]

For the remaining $(C - 1)$ occurrences, the arms in the environments exhibit the following expected rewards:
\[
e_1 : \quad \boldsymbol{\mu}_1 = (1, 1 - \delta, \delta), \quad \boldsymbol{\mu}_2 = (1 - \delta, 1, \delta), \qquad e_2 : \quad \boldsymbol{\mu}_1 = (1, \delta, 1 - \delta), \quad \boldsymbol{\mu}_2 = (\delta, 1, 1 - \delta).
\]

For $\delta \in \left(0, \frac{1}{3} \right)$, we have:
\[
\Delta_p = 1 - \delta, \qquad \Delta_{\rm minrank} = 1 - 2\delta, \qquad \Delta_{\min} = \delta.
\]

If we had used the uniform gap assumption, the regret bound would be:
\[
\mathcal{O}\left( \frac{\log T}{\Delta_{\min}^2} \right) = \mathcal{O}\left( \frac{\log T}{\delta^2} \right).
\]

However, our Algorithm \texttt{IETP\_GS} leverages the refined gap notion, yielding a regret of:
\[
\mathcal{O}\left( \frac{\log T}{\Delta_{\rm minrank}^2} \right) = \mathcal{O}\left( \frac{\log T}{(1 - 2\delta)^2} \right),
\]

which is significantly smaller when $\delta$ is small.

This example demonstrates that assuming a uniform gap does not accurately capture the variability in the refward structure. The refined gap notion introduced in this section better accommodates these differences, improving our algorithm's theoretical guarantees.






\subsubsection{Improved Regret Bound}
Combining the regret contribution from the rank discovery and the rank matching parts, we now state our improved regret upper bound for \improved.
\begin{theorem}
\label{thm:regret_bound_multiple_environment}
  Under Assumptions~\ref{a:env}, and~\ref{a:min-eigen-value}, and reward gap period $P_e(\Delta)$ (Def.~\ref{def:modified-gap}) the regret of agent $i$ when all agents run \improved is 
  \begin{align*}
      \mathbb{E}[R_T^{(i)}]  &\leq  \left(\min_{\Delta > 0}\left(\frac{64d^2L^2 \log(T)}{\kappa \Delta^2} + \sum_{e \in E} P_e(\Delta)\right)\right.\left.+ g\left(\frac{64 d^2 L^2 \log(T)}{\kappa \Delta^2_{\rm min rank}}\right)  + E N^2 + \tfrac{Nd\pi^2}{3} \right)\mu_{i,\max}.
  \end{align*}
\end{theorem}

\paragraph{Reward Gap Period:}
The regret scaling depends on how the reward gap period $P_e(\Delta)$ scales with $\Delta$. For example, let $P_e(\Delta) \leq C$, for all $e$, all $\Delta \leq \Delta_p$, and for some universal constant $C \geq 1$.  In this scenario,  
\begin{align*}
\min_{\Delta > 0}\left(\tfrac{64d^2L^2 \log(T)}{\kappa \Delta^2} + \sum_{e \in E} P_e(\Delta)\right)
\!\!\leq \tfrac{64d^2L^2 \log(T)}{\kappa \Delta_p^2}\! + E C. 
\end{align*}
Also, if there is a uniform minimum reward gap $\Delta_{\min}$ we have $P_e(\Delta_{\min}) = 0$, and we get back the guarantees in Theorem~\ref{thm:regret_bound}.

\paragraph{Regret Scaling:} Note that the regret scaling is logarithmic with $T$. We see that there are two components of regret. The first component, similar to the case of a single environment, relies on large opportunistic gaps to discover the rank and scales as $O(\tfrac{d^2 \log(T)}{\kappa \Delta_p^2})$ if gaps $\Delta_p$ or more occurs with constant gap per environment. The other regret component is due to errors in finding the latent environment in each round. Using Kendall-Tau distance based rank matching, we show that an additional $O(\tfrac{ d^2 \log(T)}{\kappa \Delta^2_{\rm min rank}})$ number of exploration suffices. As our bounds rely on  $\Delta_{\rm min rank}$ and $\Delta_p$, both of which are larger than $\Delta_{\min}$, our algorithm improves over simple adaptations of existing algorithms. The regret scales linearly with the number of latent environments $E (N^2+ \tfrac{1}{E}\sum_{e\in E} P_e(\Delta_p))$, but importantly the regret does not scale with $E$ and $T$ jointly.

\section{Piecewise Stationary Latent Parameters}
\label{sec:Piecewise Stationary Latent Parameters}
In this section, we generalize our model to capture the setting where the latent parameters $\boldsymbol{\theta}_i$ for $i \in [N]$ varies over time in a piecewise stationary manner. We assume without loss of generality that $\thetai$ lies inside the unit ball $\forall i \in [N]$. In particular, let us consider $\gamma_T$ as the number of change points in $T$ rounds, with $0 < \tau_1 < \tau_2 \dots < \tau_{\gamma_T} < T$ denote the change points. Within a particular window, $t \in [\tau_l, \tau_{l+1} - 1]$, the latent parameters $\boldsymbol{\theta}_i(l)$ for $i \in [N]$ remains fixed, while between two windows there exists at least one latent parameter that changes. 

To adapt to such a changing latent parameter, we use a meta-algorithm where we have a change detection (CD) loop at the top. Within the CD loop we execute our IETP-GS while a change detection algorithm runs in parallel. Once the change detection algorithm indicates a change we restart the loop by resetting our IETP-GS algorithm. In our case, we have $N$ users running a separate instantiation of a specific CD algorithm. 

The CD algorithm has three main functional component. First in each round a function ({$\text{CD.IsForcedExploration}(t)$}) decides whether an exploration should be forced. In a CD algorithm, typically a random variable $s$ is maintained, namely the cumulative sum (CUSUM), that is updated from the observations in the forced exploration round (see \cite{liu2018change}). In our case, $\text{CD.UpdateCUSUM}(r_i(t), \mathbf{x}_{i, m_i(t)}(t))$ represents such an update. Finally, inside the function $\text{CD.IsCUSUMOverThreshold}(t)$ a change is detected when the variable $s$ crosses a known threshold. Whenever at least one agent detects a change point it broadcasts this signal (by setting a flag $\mathcal{CD}$), and this triggers a restart of the IETP-GS.

\begin{algorithm}[t!]
\caption{CD-ETP-GS: Change Detection aided Environment-Triggered Phased Gale-Shapley (at agent $i$)}
\label{alg:alg3}
\begin{algorithmic}[1]

\State \textbf{Explore phase length:} $Block(l) := 2^l$
\State \textbf{Blackboard variable:} $\mathcal{B}$, $\mathcal{CD}$
\State \textbf{Initialize:} $D[e{:}(s_i, \sigma_i)] = \{\}$, $\tau_{\text{end}} = 0$, $l = 0$, $\hat{\tau}= 0$

\vspace{0.5em}
\For{$t' \geq 1$} 
    \State \textbf{Reset:} $\mathcal{B} \gets 1$, $\mathcal{CD} \gets 1$ \Comment{Reset across users}
    \State \textbf{Set local time:} $t \gets t' - \hat{\tau}$

    \vspace{0.2em}
     \If{CD.IsForcedExploration$(t)$} \Comment{Change point Detection Algorithm (CD)} 
        \State Play arm: $m_i(t) \gets ((i+t) mod K) + 1$
        \State Observe reward $r_i(t)$
        \State CD.UpdateCUSUM$(r_i(t), \mathbf{x}_{i, m_i(t)}(t))$
        \If{CD.IsCUSUMOverThreshold$(t)$}
            \State $\mathcal{CD} \gets \mathcal{CD} \wedge 0$ 
        \EndIf
        \State \Comment{Distributed $\mathcal{CD}$ update completes}
        \If{$\mathcal{CD} = 0$} \Comment{IETP-GS reset on change detection}
            \State Reset local time: $\hat{\tau} \gets t'$
            \State Re-initialize: $D[e{:}(s_i, \sigma_i)] = \{\}$, $\tau_{\text{end}} = 0$, $l = 0$
            \State Reset: $\hat{\theta}, \mathrm{UCB}, \mathrm{LCB}$
        \EndIf
    \Else
    \State \textbf{IETP-GS Algorithm}
        \vspace{0.2em}
        \State Same as Algorithm~\ref{alg:alg2} (IETPGS) \Comment{Environment Recovery}
        \vspace{0.2em}
        \State Same as Algorithm~\ref{alg:alg1} (ETPGS) \Comment{Exploration Triggering}
        \vspace{0.2em}
        \State Same as Algorithm~\ref{alg:alg1} (ETPGS) \Comment{Explore or Gale-Shapley Matching}
    \EndIf
\EndFor

\end{algorithmic}
\end{algorithm}

Our stated guarantees are based on three quantities of a CD algorithm: 1) expected delay of detection $\mathcal{D}$, 2) expected number of false alarm $\mathcal{F}_T$ in $T$ rounds, and 3) number of forced exploration $\mathcal{E}_T$ in $T$ rounds. Let us define the detection instants by a CD algorithm as $\hat{\tau}_1, \dots, \hat{\tau}_n$. Between the rounds $\hat{\tau}_l$ and $(\hat{\tau}_{l+1} - 1)$ a single instantiation of IETP-GS algorithm is executed. Let $\tau'$ be the first actual change point in $[\hat{\tau}_l, (\hat{\tau}_{l+1} - 1)]$ then $(\hat{\tau}_{l+1} - \tau')$ is the delay in detection. If there exists $\tilde{n} \geq 2$ detected change points inside a single window then that corresponds to $(\tilde{n} - 1)$ false alarms. Let $D_l$ be the delay in detection for the $l$-th real change point. Then the regret due to this delay adds up to $L \sum_{i =1}^{\gamma_T} D_l$.  Let $T_l$ be the window length of the $l$-th instantiation of IETP-GS. The excess regret from the forced exploration in round $l$ is typically a function of $T_l$ and we represent it as $E(T_l)$.
The regret for user $i$ can be bounded as
\begin{align*}
 R^{(i)}(T; \text{CD-ETP-GS}) &\leq L\sum_{l=1}^{\gamma_T} D_l + 
\sum_{l=1}^{\tilde{n}_l} R^{(i)}(T_l; \text{IETP-GS}) + L \times E(T_l).
\end{align*}

The delay $D_l$ is the minimum of delays among $N$ independent CD algorithm runs. More specifically, each of the CD runs are independent till the point at least one algorithm detects a change point because the forced exploration for CD is scheduled in a deterministic way. Therefore, we have $\mathbb{E}_{CD}[D_l] \leq \mathcal{D}$. Next, the false alarm frequency across $N$ users add up by using a union bound. Therefore, we have $\mathbb{E}_{CD}[\sum_l \tilde{n}_l]\leq \gamma_T + N \mathcal{F}_T$, also note that $\sum_l \tilde{n}_l \geq \gamma_T$. 

\emph{Special cases:} Finally, for any forced exploration that grows as a concave function we have due to Jensen's inequality $$\sum_{l=1}^{\tilde{n}_l} E(T_l) \leq (\sum_l \tilde{n}_l) E(\sum_{l=1}^{\tilde{n}_l}T_l / \sum_l \tilde{n}_l) \leq (\sum_l \tilde{n}_l) \mathcal{E}_{T/\gamma_T}.$$  
Similarly, due to Jensen's inequality for logarithmic regret for IETP-GS we have  $$ \sum_{l=1}^{\tilde{n}_l} R^{(i)}(T; \text{IETP-GS}) \leq (\sum_{l}\tilde{n}_l) R^{(i)}(\tfrac{T}{\sum_{l}\tilde{n}_l}; \text{IETP-GS}).$$ 

Combining all the above simplifications we get the expected regret as 
\begin{align*}
\mathbb{E}_{CD}[R^{(i)}(T; \text{CD-ETP-GS})] 
& \leq \gamma_T L \mathcal{D} + 
(\gamma_T + N\mathcal{F}_T) R^{(i)}(\tfrac{T}{\gamma_T}; \text{IETP-GS}) + L (\gamma_T + N\mathcal{F}_T) \mathcal{E}_{T/\gamma_T}.    
\end{align*}

For example, one can adapt the Two-sided CUSUM algorithm in \cite{liu2018change} to our linear contextual setup. In particular, using a threshold of $h$ and forced exploration rate of $\alpha$ (i.e. exploration rate $\mathcal{E}_{T} = O(\alpha T)$) one can obtain a mean detection delay $\mathcal{D} = O(h/\alpha)$,  and false alarm rate $\mathcal{F}_T = O(T\exp(-h))$. Note that parameters $d$ and $\kappa$ are omitted in the above dependence. Making a choice of $h = \Theta(\log(NT/\gamma_T))$ and $\alpha = \Theta(\sqrt{\gamma_T/T \times \log(NT/\gamma_T)})$ we obtain 
\begin{align*}
\mathbb{E}_{CD}[R^{(i)}(T; \text{CD-ETP-GS})] 
=O\Big( L \sqrt{\gamma_T T \times \log(NT/\gamma_T)} + 
2\gamma_T  R^{(i)}(\tfrac{T}{\gamma_T}; \text{IETP-GS}) \Big).    
\end{align*}
Note that we require $O(1/\sqrt{T})$ rate of forced exploration to capture change points quickly, which leads to the dominating $O(\sqrt{T})$ term in the regret bound. We leave a rigorous adaptation of change detection algorithms (e.g. \cite{liu2018change,change-point-kl,hou2024changepoint,huang2025}) to our setup as a future work.



\section{Conclusion and Future Work}
\label{sec:conclusion}
In this paper, we address the problem of competitive contextual bandits in a linear setup. Based on the contexts, we show that we first need to estimate a latent variable (denoting environment) and then perform the classical Gale Shapley algorithm. We obtain logarithmic regret scaling for our proposed algorithm. As an immediate next step, we want to study this problem where the number of environments can be arbitrarily large (even continuous). We would also want to remove assumptions (like the spectral assumption) on the feature vectors. Additionally, our framework can be extended to scenarios where the latent vector of agents varies across environments. We want to emphasize that there are a lot of open problems in the \emph{bandits and markets} literature; for example the problem of two-sided learning where the preferences of the agent as well as the arm side need to learn would become interesting with contextual information. We want to keep these as our future endeavors.

\bibliographystyle{apalike} 
\bibliography{arxiv}
\onecolumn
\appendix
\def \event{\mathcal{E}_3(t)}
\section{Notations}
\begin{table}[htb]
	\centering
	\begin{tabular}{l|l}
		Notation & Meaning \\
		\hline 
		$\explore$ & Number of exploration rounds till time t\\
		$[N]$ & {1,2,\dots,N}\\
		$N$ & Number of Agents\\
		$K$ & Number of Arms\\
  $\mathcal{E}$ & Finite set of environments with cardinality $E$.\\
		$\boldsymbol{x}_{ij}(t)$ & Feature vector of arm $j$ as seen by agent $i$ at time $t$\\
		$\boldsymbol{\theta}_{i}$ & Latent vector corresponding to agent $i$\\
		$\mu_{ij}(t)= \langle\thetai,\feature \rangle$ & Mean reward obtained when agent $i$ plays arm $j$ at time $t$\\
		$m_i(t)$ & Arm matched to player $i$ at time $t$\\
		$d$ & Dimension of the latent and feature vectors\\
		$T$ & Horizon \\
		$\et$ & Least Squares Estimate of $\theta_i$ at time $t$ \\
		$r_{i}(t)$ & Reward obtained by agent $i$ in round $t$\\
        $\rho_i^e \in \mathbb{R}^K$ & preference vector of agent $i$ over arms in environment $e$.\\
		$\eta_i(t)$ & 1-Subgaussian noise added to the reward of agent $i$ at time t \\
		$\lambda_{\rm{min}}(A)$ & Minimum eigenvalue of Matrix A which has real eigenvalues \\
		$\| \boldsymbol{x} \|_A $ & $\sqrt{\boldsymbol{x}^T A \boldsymbol{x}}$ for a positive definite matrix A\\
		$\Delta_{\min, i}(t)$ & $\min_{j,j'\in \mathrm{Top}_i(N+1), j \neq j'} |\langle \boldsymbol{x}_{i,j}(t) - \boldsymbol{x}_{i, j'}(t), \boldsymbol{\theta}_i\rangle|$\\
		$\mathrm{Top}_i(N+1)$ & Set of indices of top $N+1$ arms for agent $i$ \\
		$\tau_i$ &  $ \min\left\{t: \Delta_{\min, i}(t) \geq  \tfrac{8dL}{\sqrt{\kappa}}\sqrt{\tfrac{\log(t)}{t}}\right\}$ \\
		$e(t)$ &  Environment active at time $t$\\
        $\Delta_{i, {\rm rank}}(t)$ & $\underset{e \neq e(t)}{\min}  \max \{ |{\mu}_{ij}(t) - {\mu}_{ij'}(t)| :(j,j') \in {\rm Inv}(\rho_{i}^e[1:N], \rho_i^{e(t)}[1:N])\}  $ \\
		$\Delta_{\rm min rank}$ &   $\min_{1 \leq t \leq T}\hspace{0.5em}\min_{i \in [N]}  \Delta_{i, {\rm rank}}(t) $  \\
		$\rm pr_{i}(t)$ &  Partial ranking of arms made by agent $i$ at time $t$ restricted to top $N$ preferences \\
		$\sigma_{i}[e]$ &  True preference ranking of top $N$ arms of agent $i$ in environment $e$\\
  $m_i^*(e)$ & Stable matched arm of player $i$ in environment $e$\\
  $\mu_{\rm i,max}$& $\max_{t,e \in \mathcal{E}}\mu_{i,m^*_i(e(t))}(t)$\\
	\end{tabular}
	\caption{Table of Notations}
	\label{tab:notations}
\end{table}

\newpage 
 
\section{Proof of Theorem \ref{thm:regret_bound}}
$\bullet$ \textbf{Proof Sketch:}
\begin{enumerate}
    \item \textbf{Least Square Estimate:} We obtain the least square estimate $\et$ of $\thetai$ and establish a concentration bound for the error between estimated mean $\hat{\mu}_{i,j}(t)$ and true mean $\mu_{i,j}(t)$. To do this, we would need to require an upper bound on the matrix weighted norm $\|\boldsymbol{x}\|_{\vi}$, which is established in Lemma ~\eqref{upper_bound_norm}.
    \item \textbf{Good Event Construction:} We construct the good event $\mathcal{G}$ which holds with high probability. This event ensures that the confidence intervals around the estimated mean are shrinking which is crucial in identification of true ranking..

    \item \textbf{Intermediate Lemmas:} Several intermediate lemmas are employed to bridge the gap between the concentration bounds and the final regret bound. Lemma~\eqref{lem:ucb_lcb} ensures that the ranking derived from the UCB and LCB estimates aligns with the true ranking of the arms. Lemma~\eqref{lemma:del_resolvable} determines whether the top \( N \) arms for each agent can be distinguished based on the confidence intervals obtained from exploration. Lemma~\eqref{lem:bad_event_probability} provides an upper bound on the number of steps during which the desired conditions (good events) do not hold. Lastly, Lemma~\eqref{max_steps_in_gale_shapley} offers an upper bound on the number of steps required to reach a stable matching once all agents have identified their preferences and engage in the Gale-Shapley algorithm.

	\item \textbf{Regret Bound:} Finally, we use the intermediate lemmas to derive the regret bound.
\end{enumerate}

\subsection{A Concentration Bound}
Recall that  \begin{align}
  \v & = \sum_{s=1}^{{t}}
  \boldsymbol{x}_{i,m_i(s)}({s}) \boldsymbol{x}_{i,m_i(s)}^T(s) \notag \\
\end{align}

\begin{lemma}
	Under Assumption \eqref{a:min-eigen-value} and round robin exploration, we have a concentration bound for any vector $\boldsymbol{x}$ satisfying $\|\boldsymbol{x}\|\leq L$ and $t \geq d $ , we have 
	\begin{align*}
		\|\boldsymbol{x}\|_{\vi} \leq \sqrt{L \left( \frac{d }{\kappa \explore} \right)}.
	\end{align*}
	
	\label{upper_bound_norm} 
 where $\explore$ denotes the number of exploratory rounds till time $t$
\end{lemma}
\begin{proof}
Our aim here is to find a lower bound on the minimum eigenvalue of $\vi$.	 We will make use of the fact that if \( A \) and \( B \) are two symmetric matrices, then
	\[
	\lambda_{\min}(A+B) \geq \lambda_{\min}(A) + \lambda_{\min}(B).
	\]
  Without loss of generality, focus on the time steps corresponding to the $\explore$ exploratory steps and lower bound the minimum eigenvalue contribution from the remaining $t-\explore$ terms as 0. This follows as $\boldsymbol{x}\boldsymbol{x}^T$ is a rank one matrix. The \( \explore\) feature vectors corresponding to the explored arms can be partitioned into groups of \(d\) feature vectors each, such that no two feature vectors in a group are identical.There would be total of \(\left\lfloor \frac{\explore}{d} \right\rfloor\) such groups and one remainder group containing \(\explore - d \left\lfloor \frac{\explore}{d} \right\rfloor\) feature vectors.Let the time indices corresponding to the exploration time steps be denoted by \( u_1, u_2, \ldots, u_{\explore} \).

\begin{align*}
  &\lambda_{\min}(\v)\geq \lambda_{\min}\left(
\sum_{s=1}^{\explore} \boldsymbol{x}_{i,\isk}(u_s) \boldsymbol{x}_{i,\isk}^T(u_s) \right) \\
 &= \lambda_{\min}\left( \sum_{b=1}^{\left\lfloor \frac{\explore}{d} \right\rfloor} \sum_{s=(b-1)d+1}^{b d}\!\! \boldsymbol{x}_{i,\isk}(u_s) \boldsymbol{x}_{i,\isk}^T(u_s)  +\!\!\! \sum_{s=\explore - (\explore \bmod d) + 1}^{\explore} \!\!\!\!\!\!\boldsymbol{x}_{i,\isk}(u_s) \boldsymbol{x}_{i,\isk}^T(u_s)  \right)\\
 &\geq \lambda_{\min}\left( \sum_{b=1}^{\left\lfloor \frac{\explore}{d} \right\rfloor} \sum_{s=(b-1)d+1}^{b d}\!\! \boldsymbol{x}_{i,\isk}(u_s) \boldsymbol{x}_{i,\isk}^T(u_s)     \right)
\end{align*}

From Assumption \eqref{a:min-eigen-value}, we have that the minimum singular value of the outer product of any \( d \) features is at least \( \kappa \). Thus, for any complete set of \( d \) features, we have:

\[
\lambda_{\min} \left( \sum_{s=(b-1)d+1}^{bd} \boldsymbol{x}_{i,\isk}(s) \boldsymbol{x}_{i,\isk}^T(s) \right) \geq \kappa.
\]

Since there are \( \left\lfloor \frac{\explore}{d} \right\rfloor \) such complete sets, the sum of the minimum eigenvalues over these sets is at least:

\[
\sum_{b=1}^{\left\lfloor \frac{\explore}{d} \right\rfloor} \lambda_{\min} \left( \sum_{s=(b-1)d+1}^{bd} \boldsymbol{x}_{i,\isk}(s) \boldsymbol{x}_{i,\isk}^T(s) \right) \geq \kappa \left\lfloor \frac{\explore}{d} \right\rfloor.
\]
    
 \[
\lambda_{\min} \left(\vi \right) \geq \kappa \left\lfloor \frac{\explore}{d} \right\rfloor.
\]

Thus, we have that
\begin{align*}
    \|\boldsymbol{x}\|_{V_{i}(t)^{-1}}^{2} &= \boldsymbol{x}^T (V_{i}(t)^{-1}) \boldsymbol{x} \\
    &\leq L^2 \cdot \lambda_{\max}(V_{i}(t)^{-1}) \\
    &= L^2 \cdot \frac{1}{\lambda_{\min}(V_{i}(t))} \\
    &\leq L^2 \cdot \frac{1}{ \kappa \left\lfloor \frac{\explore}{d} \right\rfloor} \\
    &\leq L^2 \cdot \frac{d}{ \kappa \explore }
\end{align*}

The floor operation has been omitted as the discretization error becomes negligible for sufficiently large  $\explore$
\end{proof}

\subsection{Least Square Estimate}
Recall that reward received at time $t$ is $r_{i}(t)=\langle\thetai,\boldsymbol{x}_{i,m_i(t)}(t)\rangle+\eta_{t}$ where $\eta_{t}$ is zero-mean sub-gaussian noise.
The least squares estimate of $\thetai$ is the minimizer of the following loss function. 
\begin{align*}
  \loss &= \sum_{s=1}^{t} \left(r_i(s) - \langle \boldsymbol{\theta}, \boldsymbol{x}_{i,m_i(s)}(s) \rangle \right)^2, \\ 
  \nabla_\theta \loss &= 0 \implies  2 \sum_{s=1}^{t} \left( r_i(s) - \langle \boldsymbol{\theta}, \boldsymbol{x}_{i,m_i(s)}(s) \rangle \right) \boldsymbol{x}_{i,m_i(s)}(s)=0
\end{align*}

The minimizer is given as 
\begin{equation}
	\et = \vi\sum_{s=1}^{t} r_i(s)  
 \boldsymbol{x}_{i,m_i(s)}(s)
	\label{eq:theta_estimated}
\end{equation}

When comparing \(\thetai\) and \(\et\) in the direction of \(x \in \mathbb{R}^d\), we obtain

\begin{align*}
	\langle \et - \thetai, \boldsymbol{x} \rangle &= \left\langle \boldsymbol{x}, \vi \sum_{s=1}^{t} \boldsymbol{x}_{i,m_i(s)}(s) r_i(s) - \thetai \right\rangle \\
	&= \left\langle \boldsymbol{x}, \vi \sum_{s=1}^{t} \boldsymbol{x}_{i,m_i(s)}(s) \left(\boldsymbol{x}_{i,m_i(s)}(s)^T \thetai + \eta_s \right) - \thetai \right\rangle \\
	&= \left\langle \boldsymbol{x}, \vi \sum_{s=1}^{t} \boldsymbol{x}_{i,m_i(s)}(s) \eta_s \right\rangle \\
	&= \sum_{s=1}^{t} \left\langle \boldsymbol{x}, \vi \boldsymbol{x}_{i,m_i(s)}(s) \right\rangle \eta_s .
\end{align*}
If X is \(\sigma\)-subgaussian, then for any \(\delta \in (0,2]\),
\[ P(|X| \geq \sqrt{2\sigma^2 log(\frac{2}{\delta})}) \leq \delta.  \]

\begin{equation*}
	\mathbb{P} \left( \left|\langle \bx, \et - \thetai \rangle\right| \geq \sqrt{2 \sum_{s=1}^{t} \left\langle \boldsymbol{x}, \vi \boldsymbol{x}_{i,m_i(s)}(s) \right\rangle^2 \log \left( \frac{2}{\delta} \right)} \right) \leq \delta .
\end{equation*}
It can be shown that
\begin{align*}
	\sum_{s=1}^{t} \left\langle \boldsymbol{x}, \vi \boldsymbol{x}_{i,m_i(s)}(s) \right\rangle^2 &= \|\boldsymbol{x}\|_{\vi}^2\\ 
\end{align*}

\begin{equation}
	\mathbb{P} \left( \left| \langle \boldsymbol{x}, \et - \thetai \rangle \right| \geq \sqrt{2 \|\boldsymbol{x}\|^2_{\vi} \log \left( \frac{2}{\delta_t} \right)} \right) \leq \delta_t .
\end{equation}

Putting $\delta_t$ as $\tfrac{2}{t^2}$ in the above equation we get

\begin{equation}\label{final_bound}
	\mathbb{P} \left( \left| \langle \boldsymbol{x}, \et - \thetai \rangle \right| \geq \sqrt{2 \|\boldsymbol{x}\|_{\vi}^2 \log \left( {t^2 }\right)} \right) \leq \frac{2}{t^2} .
\end{equation}

\subsection{Good Event Construction}

Recall that the set $\{v_{\ell}: \ell \in [d]\}$   forms a set of  {standard orthonormal basis} vectors for $\mathbb{R}^d$. Define  
\begin{align*}
    \mathcal{G}(t) &:=   \left\{ \begin{array}{l}
        \left|\langle \boldsymbol{v}_{\ell},\et - \thetai \rangle \right| \leq \sqrt{2 \|\boldsymbol{v}_{\ell}\|_{\vi}^2 \log \left( t^2 \right)}
    \end{array} \forall \ell \in [d], \forall i \in [N]  \right\}.\\
\end{align*}
Define the Good event as
\[  \mathcal{G}:= \bigcap_{t} \mathcal{G}(t)\]

If  $\mathcal{G}(t)$ holds, we have for $\forall i$  $\in [N]$  and any vector $\boldsymbol{x}$
\begin{align*}
	\left|\langle \boldsymbol{x}, \et - \thetai \rangle \right|
	& = \left|\langle\sum_{\ell\in [d]}\langle \boldsymbol{x}, \boldsymbol{v}_{\ell}\rangle \boldsymbol{v}_{\ell}, \et - \thetai \rangle \right|\\
	&\leq \sum_{\ell\in [d]}|\langle \boldsymbol{x}, \boldsymbol{v}_{\ell}\rangle| \left|\langle \boldsymbol{v}_{\ell}, \et - \thetai \rangle \right|\\
	& \leq \sum_{\ell\in [d]}|\langle \boldsymbol{x}, \boldsymbol{v}_{\ell}\rangle| \sqrt{2 \|\boldsymbol{v}_{\ell}\|_{\vi}^2 \log \left( t^2 \right)}\\
\end{align*}
This motivates us to define the confidence width as
\begin{align*}
	w_{i,j}(t) :&= \sum_{\ell\in [d]}|\langle \boldsymbol{x}_{i,j}(t), \boldsymbol{v}_{\ell}\rangle| \sqrt{2 \|\boldsymbol{v}_{\ell}\|_{\vi}^2 \log \left( t^2 \right)} \\ 
\end{align*}
We define UCB and LCB for the estimated mean as 
\begin{align*}
	\ucb_{i,j}(t) &= \hat{\mu}_{ij}(t) + w_{i,j}(t)\\ 
	\lcb_{i,j}(t) &= \hat{\mu}_{ij}(t) - w_{i,j}(t) 
\end{align*}
\begin{lemma}
	\label{lem:confidence_width}
	Conditional on $\mathcal{G}(t)$, the confidence width $w_{i,j}(t)$ is at most $\frac{2dL}{\sqrt{\kappa}} \sqrt{\frac{ \log \left( t  \right)}{\explore }}$.
\end{lemma}
\begin{proof}

We can establish the following upper bound on the confidence width 
\begin{align*}
	\sum_{\ell\in [d]}|\langle \bx, \bv_{\ell}\rangle| \sqrt{2 \|\bv_{\ell}\|_{\vi}^2 \log \left( t^2 \right)} 
	& \leq \sum_{\ell\in [d]}|\langle \bx, \bv_{\ell}\rangle| \sqrt{\frac{2}{\lambda_{\min}(V_t(i))}\log \left( t^2 \right)} \\
	& \leq \sum_{\ell\in [d]}|\langle \bx, \bv_{\ell}\rangle| \sqrt{\frac{2d}{\kappa \explore }\log \left( t^2 \right)} \\
	& = \sum_{\ell\in [d]}|  \bx_l| \sqrt{\frac{2d}{\kappa \explore }\log \left( t^2 \right)} \\
	& = \|\bx\|_1 \sqrt{\frac{2d}{\kappa \explore }\log \left( t^2 \right)} \\
	& \leq \sqrt{d}\|\bx\|_2 \sqrt{\frac{2d}{\kappa \explore }\log \left( t^2 \right)} \\
\end{align*}
Thus if $\mathcal{G}(t)$ holds, then we have
\begin{align}
	\left|\langle \bx, \et - \thetai \rangle \right|
	& \leq \sum_{\ell\in [d]}|\langle \bx, \bv_{\ell}\rangle| \sqrt{2 \|\bv_{\ell}\|_{\vi}^2 \log \left( t^2 \right)}
	 \leq  \|\bx\|_2 \sqrt{\frac{2d^2}{\kappa \explore }\log \left( t^2 \right)} \label{eq:good_event_bound}
\end{align}
\end{proof}
Denote $w_{\max}(t)= L \sqrt{\frac{2d^2}{\kappa \explore }\log \left( t^2 \right)}$


\subsection{Intermediate Lemmas used in proof of Theorem \texorpdfstring{\eqref{thm:regret_bound}}{(Theorem Ref)}}

\begin{lemma}
	Conditional on $\mathcal{G}(t)$, $\ucb_{ij}(t)<\lcb_{ij'}(t) \implies \mu_{ij}(t)<\mu_{ij'}(t')$.
	\label{lem:ucb_lcb}
\end{lemma}
\begin{proof}
	\begin{align*}
		\ucb_{ij}(t)<\lcb_{ij'}(t) \iff& \hat{\mu}_{ij}(t) + w_{ij}(t)<\hat{\mu}_{ij'}(t) - w_{ij'}(t) \\
		& \mu_{ij}(t)\stackrel{(a)}{\leq} \hat{\mu}_{ij}(t) + w_{ij}(t) < \hat{\mu}_{ij'}(t) - w_{ij'}(t) \stackrel{(b)}{\leq} \mu_{ij'}(t)
	\end{align*} 
	where (a) and (b) follows from Equation \eqref{eq:good_event_bound}

	We also have the following logical equivalence: $A \implies B \iff \neg B \implies \neg A$

	Thus $\left(\ucb_{ij}(t)<\lcb_{ij'}(t) \implies \mu_{ij}<\mu_{ij'}\right) \iff \left(\mu_{ij} \geq \mu_{ij'} \implies \ucb_{ij}(t)\geq \lcb_{ij'}(t)\right)$
\end{proof}

\begin{lemma}
	Conditional on $\mathcal{G}:=\bigcap_{t} \mathcal{G}(t)$, Agent $i$ can distinguish the top $N$ arms of the environment which showed up if $\Delta_{\min,i}(t) \geq 8Ld\sqrt{\frac{\log(t)}{\kappa \explore}}$.
	\label{lemma:del_resolvable}
\end{lemma}

\begin{proof}
Recall that $\Delta_{\min,i}(t) = \min_{j,j' \in \mathrm{Top}_i(N+1)} |\mu_{ij}(t)-\mu_{ij'}(t)|$.

Let $j, j' \in \mathrm{Top}_i(N+1)$ be such that $\mu_{ij} < \mu_{ij'}$ .The arms cannot be resolved iff $\ucb_{ij}(t) > \lcb_{ij'}(t)$. Assume they cannot be resolved at time $t$ then,
\begin{align}
	\lcb_{ij'}(t) &<\ucb_{ij}(t) \notag \\
	\hat{\mu}_{ij'}(t) - w_{ij'}(t) &< \hat{\mu}_{ij}(t) + w_{i,j}(t) \notag\\ 
	{\mu}_{ij'} - 2w_{ij'}(t) &< \mu_{ij} + 2w_{ij}(t) \notag \\
	\mu_{ij'}(t) - \mu_{ij}(t) &< 2w_{ij}(t) + 2w_{ij'}(t) \notag \\
	\mu_{ij'}(t) - \mu_{ij}(t) &\stackrel{(a)}{<} 4 L \sqrt{\frac{2 d^2}{\kappa \explore}\log \left( t^2 \right)} \notag \\       
	\mu_{ij'}(t) - \mu_{ij}(t) &\leq 8Ld\sqrt{\frac{\log(t)}{\kappa \explore}} \notag \\
\end{align}
 where (a) follows from Equation ~\eqref{eq:good_event_bound}.
 Thus if $\mu_{ij'}(t) - \mu_{ij}(t) \geq 8Ld\sqrt{\frac{\log(t)}{\kappa \explore}} $  then the arms can be separated. If all arms in $\mathrm{Top}_i(N+1)$ are resolved then we have $\Delta_{\min,i}(t) \geq 8Ld\sqrt{\frac{\log(t)}{\kappa \explore}}$.

	 
\end{proof}

\begin{lemma}(Upper Bounding number of bad rounds)
	$  \mathbb{E}\left[\sum_{t=1}^{T} \mathbf{1}\{\neg \mathcal{G}(t)\}\right] \leq \frac{Nd\pi^2}{3}$
	\label{lem:bad_event_probability}
\end{lemma}
\begin{proof}
	
	Denote the bad event as $\mathcal{B}.$
	\[\mathcal{B} := \neg \mathcal{G} = \bigcup_{t} \neg \mathcal{G}(t) \]
	\[ P(\mathcal{B}) {\leq} \sum_{t} P(\neg \mathcal{G}(t))= \mathbb{E}\left[\sum_{t=1}^{T} \mathbf{1}\{\neg \mathcal{G}(t)\}\right]  \]
\begin{equation}
  \begin{aligned}
      \mathbb{E}\left[\sum_{t=1}^{T} \mathbf{1}\{\neg \mathcal{G}(t)\}\right] 
      &= \mathbb{E}\left[\sum_{t=1}^{T} \mathbf{1}\left\{\exists i \in [N], \ell \in [d] : 
	  \left|\langle \bv_{\ell}, \et - \thetai \rangle \right| > 
	  \sqrt{2 \|\bv_{\ell}\|_{\vi}^2 \log \left( t^2 \right)} 
      \right\} \right] \\
      &\stackrel{(a)}{\leq} \sum_{i \in [N], l \in [d]} \mathbb{E}\left[\sum_{t=1}^{T} \mathbf{1}\left\{
		  \left|\langle \bv_{\ell}, \et - \thetai \rangle \right| > \|x_j\|_{\vi} \sqrt{2 \log \left(t^2 \right)}
		  \right\}\right] \\
		  &\stackrel{(b)}{\leq} \sum_{i \in [N], l \in [d]} \sum_{t=1}^{T} \frac{2}{t^2 } \\
		  &= \frac{Nd \pi^2}{3}.
		\end{aligned}
	\end{equation}
  where step (a) and (b) follows from  Union bound and Linearity of Expectation respectively
\end{proof}

Let us define an event that characterizes whether all agents have successfully established the ranking for the top-$N$ arms corresponding to the environment that has appeared. Specifically, define  
\[
\event = \bigcap_{i \in [N]} \left(\sigma_i(t) \neq \phi\right),
\]  
where $\sigma_i(t)$ is as defined in Algorithm \ref{alg:alg1}, Line 6. When $\event$ holds, it implies that all agents play according to the identified environment, i.e., they execute the Gale-Shapley algorithm until convergence and then exploit their stable matched arms. Conversely, when $\event$ does not hold, the agents continue performing round-robin exploration.  

\begin{lemma}  
	\label{lemma: dominant regret term}
	Conditional on good event $\mathcal{G}$ the event $\neg \event$ can occur only if  
	\[
	\mathcal{T}(t) < \frac{64 d^2 L^2 \log T}{\kappa \Delta_{\min}^2}.
	\]  
\end{lemma}  

\begin{proof}  
	If $\neg \event$ holds, at least one agent has failed to identify the environment. By the contrapositive of Lemma \ref{lemma:del_resolvable}, this can happen only if  
	\[
	\mathcal{T}(t) < \frac{64 d^2 L^2 \log T}{\kappa \Delta_{\min}^2}.
	\]  
\end{proof}  
\begin{lemma}
	At most $N^2 - 2N + 2$ steps are required to reach an agent-optimal stable matching when agents play the Gale-Shapley algorithm.
	\label{max_steps_in_gale_shapley}
\end{lemma}
	
\begin{proof}
	Agents will start playing the Gale-Shapley Algorithm when they have identified their top $N$ arms. Lemma ~\eqref{lem:ucb_lcb} guarantees that the ranking obtained using UCB and LCB is consistent with the true ranking of the arms.  
	
	The maximum number of steps required to reach a stable matching when both sides have the same number, \( N = M \), was established in \cite{gale1962college}. We now demonstrate that the maximum number of steps to reach a stable matching when \( M > N \) remains the same as in the case where \( N = M \).

	Once an arm becomes matched, it remains matched to some agent throughout the process. Suppose agent \(i\) is matched to arm \(j\). Since agent \(i\) proposed to arm \(j\) only after being rejected by all its higher-ranked preferences, agent \(i\) will not propose to any other arm unless arm \(j\) unmatches it. However, arm \(j\) will always remain matched, potentially to a different agent it prefers more. This ensures that every arm stays matched as soon as it receives a proposal.
	
	Now, to show that no agent can be rejected by all of its top \(N\) preferences, assume, for the sake of contradiction, that agent one is rejected by its top \(N\) preferred arms. This would imply that all \(N\) arms have been matched with agents they prefer over agent 1. However, since there are only \(N - 1\) other agents, it is impossible for all \(N\) arms to prefer another agent over agent 1. Therefore, agent one must be accepted by at least one of its top \(N\) preferences, which contradicts the assumption.
	
	Thus, each agent can be rejected at most \(N - 1\) times, leading to at most \(N - 1\) rejections per agent. In the worst case, each agent makes \(N - 1\) proposals before finding a match, meaning the total number of proposals is bounded by \((N - 1) \times (N - 1) + N = N^2 - N + 1\).
	
	With \(N\) proposals made in the first round, and at most one proposal per agent in each subsequent round, the total number of steps is upper bounded by \(N^2 - N + 1 - (N - 1) = N^2 - 2N + 2\).
	\end{proof}

	\begin{lemma}
		\label{lemma: net GS}
		Conditional on the good event $\mathcal{G}$ and the event $\event$, it will take atmost $N^2 - 2N + 2$ steps for all agents to reach an agent-optimal stable matching.
	\end{lemma}
	\begin{proof}
		Give $\mathcal{G}$ and $\cap_{t}\event$ the agents would be always able to identify environment correctly and play Gale Shapley algorithm until convergence. From Lemma \eqref{max_steps_in_gale_shapley} we know that at most $N^2 - 2N + 2$ steps are required to reach an agent-optimal stable matching when agents play the Gale-Shapley algorithm.
	\end{proof}

\subsection{Regret}
\label{proof: delta_min}
Regret of player $i$ by time $T$ is given as 
\begin{align*}
	  \mathbb{E}[R_T^{(i)}]=\mathbb{E}\left[\sum_{t=1}^{T} \left( \mu_{i, m_i^{*}(e)}(t) -\mu_{i,m_i(t)}(t) \right)\right],
\end{align*}
where $m_i^*(e)$ is the stable matching of player $i$ in the environment $e$ and the expectation is over the stochasticity of rewards and algorithm played by the agents 
\begin{align*}
    \mathbb{E}\left[R_i(T)\right] &\leq \sum_{t=1}^{T} \mathbb{E} \left[ \mathbf{1}\{m_i(t) \neq m_i^*\} \right] \Delta_{i,\text{max}} \\
    &\leq \sum_{t=1}^{T} \mathbb{E} \left[ \mathbf{1}\{m_i(t) \neq m_i^*\}  | \cap_t\mathcal{G}(t) \right] \Delta_{i,\text{max}} 
    + \sum_{t=1}^{T} \mathbb{E} \left[ \mathbf{1}\{\neg \mathcal{G}(t)\}   \right] \Delta_{i,\text{max}} \\
    &{\leq} \sum_{t=1}^{T} \mathbb{E} \left[ \mathbf{1}\{m_i(t) \neq m_i^*\} | \cap_t\mathcal{G}(t) \right] \Delta_{i,\text{max}} 
    + \sum_{t=1}^{T} \mathbb{E} \left[ \mathbf{1}\{\neg \mathcal{G}(t)\}   \right] \Delta_{i,\text{max}} \\
    &{\leq} \underbrace{\sum_{t=1}^{T} \mathbb{E} \left[ \mathbf{1}\{m_i(t) \neq m_i^*\}  \cap \event | \cap_t\mathcal{G}(t)\right] \Delta_{i,\text{max}}}_{\text{ Lemma \ref{lemma: net GS}}} +\underbrace{\sum_{t=1}^{T} \mathbb{E} \left[ \neg\event | \cap_t\mathcal{G}(t) \right] \Delta_{i,\text{max}}}_{\text{Lemma \ref{lemma: dominant regret term}}} 
    + \underbrace{\sum_{t=1}^{T} \mathbb{E} \left[ \mathbf{1}\{\neg \mathcal{G}(t)\}   \right] \Delta_{i,\text{max}}}_{\text{Lemma \ref{lem:bad_event_probability}}} \\
    &{\leq} \left( EN^2 +\frac{64 d^2 L^2 \log T}{\kappa \Delta_{\min}^2}
    + \frac{Nd \pi^2}{3} \right)\Delta_{i,\text{max}} \\
\end{align*}

\section{Proof of Theorem \texorpdfstring{\eqref{thm:regret_bound_multiple_environment}}{(Theorem Ref)}}

\label{appendix:proof_regret_multiple_environments}

\textbf{Proof Sketch:}The proof follows a structure similar to that of Theorem \eqref{alg:alg1}. We provide a concise outline of the key steps. 
The Good event construction and least square estimate $\et$ remain the same as in 
previous proof, we define the good event that forms the basis of our analysis. We then derive an upper bound on the time required for all agents to identify their preferences. Following this, we establish an upper bound on the additional exploration needed for agents to correctly identify the environment showing up after determining their preferences. Once agents are able to identify the environment correctly and have found the preferences correctly, they start playing the Gale-Shapley algorithm based on the environment showing up. Finally, we combine these components to obtain an overall bound on the regret.

\subsection{Good Event Construction}
 We define the Good event in the same way as in the previous section.The set $\{v_{\ell}: \ell \in [d]\}$   forms a set of  {standard orthonormal basis} vectors for $\mathbb{R}^d$. Define the event $\mathcal{G}(t)$ as 
 
\begin{align*}
    \mathcal{G}(t) :=  \left\{
        \left|\langle \boldsymbol{v}_{\ell}, \et - \thetai \rangle \right| \leq \sqrt{2 \|\boldsymbol{\bv}_{\ell}\|_{V_t(i)^{-1}}^2 \log \left( t^2 \right)} 
        \quad \forall \ell \in [d], \forall i \in [N]
        \right\} 
\end{align*}

We define the Good event as $\mathcal{G}:= \bigcap_{t} \mathcal{G}(t)$.

\subsection{Characterizing Exploration needed to identify the preferences}
\label{sec:characterizing_exploration}
Recall that we defined $\nu_{i,e}$ as the occurrence by which agent $i$ has identified its preference over the arms in environment $e$. Let $\mathcal{T}(t_e(\nu))$ denote the number of exploration by agent $i$ up to the $\nu$-th appearance of the environment $e$. 
Then we have 
\begin{align*}
	\nu_{i,e} = \min\left\{\nu: \Delta_{\min, i}(t_e(\nu)) \geq  \tfrac{8dL}{\sqrt{\kappa}}\sqrt{\tfrac{\log(t_e(\nu))}{\mathcal{T}(t_e(\nu))}}\right\},
\end{align*}
Note that as we explore synchronously we have $\mathcal{T}(t_e(\nu))$ to be uniform across all agents. 

We first recall the definition 
$$
P_e(\Delta) \triangleq \max_{i \in [N]} \left(\min \{ \delta\nu_e: \delta\nu_e \geq 0, \Delta_{\min, i}(t_e(\nu_e + \delta\nu_e)) \geq \Delta\}\right).
$$

Let $\nu'_{e}(\Delta)$ as the minimum number of occurrence of environment $e$ on or before we have $\mathcal{T}(t) = \tau(\Delta)$. We want to upper bound   $\max(0, \nu_{i,e} - \nu'_{e}(\Delta))$. This denotes the number of occurrence of environment $e$ required by agent $i$ to discover the rank $\sigma[e]$ post $\tau(\Delta)$ many exploration is completed

We want to upper bound $\nu_{i,e}$. If $\nu_{i,e}\leq \nu'_{e}(\Delta)$ then we have that as an upper bound. Let us consider the situation when $\nu_{i,e} > \nu'_{e}(\Delta)$. 
\begin{align*}
\nu_{i,e} &\leq \min_{\Delta > 0}\hspace{0.2em} \min \{\nu: \mathcal{T}(t_e(\nu)) \geq \tau(\Delta), \Delta_{\min, i}(t_e(\nu)) \geq \Delta\}\\
&\leq \min_{\Delta > 0}\hspace{0.2em} \min \{\nu: \mathcal{T}(t_e(\nu'_{e}(\Delta))) + \nu - \nu'_{e}(\Delta)\geq \tau(\Delta), \Delta_{\min, i}(t_e(\nu)) \geq \Delta\}\\
&\leq \min \{\nu: \nu - \nu'_{e}(\Delta) \geq 0, \Delta_{\min, i}(t_e(\nu)) \geq \Delta\}, \\
&= \min \{\nu'_{e}(\Delta) + \delta\nu: \delta\nu\geq 0, \Delta_{\min, i}(t_e(\nu'_{e}(\Delta) + \delta\nu)) \geq \Delta\}, \\
&\leq  \nu'_{e}(\Delta) + P\big(\Delta\big), 
\end{align*}
The second inequality is true because, each time environment $e$ is encountered, and top-$N$ rank of $e$ is not resolved by agent $i$ then either (i) there is an exploration phase ongoing, or (ii) a new exploration phase is triggered. In either case, at least $(\mathcal{T}(t_e(\nu'_{e}(\Delta))) + \max(0, \nu - \nu'_{e}(\Delta)) - 1)$ many explorations by the time $e$ is encountered for the $\nu$-th time without the rank discovery of environment $e$. The third inequality uses the fact that $\mathcal{T}(t_e(\nu'_{e}(\Delta))) \geq \tau(\Delta)$. The final inequality follows from the definition of $P_e(\Delta)$. 

Therefore, we have $\max(0, \nu_{i,e} - \nu'_e(\Delta)) \leq P_e(\Delta)$. So we only have a constant regret added per environment after sufficient exploration is completed.
 
Denote by $\nu_e$ the extra occurrence of environment $e$ needed, after the $\nu'_{e}(\Delta)$-th occurrence, by which all agents have identified their preferences for environment $i$. We can upper bound $\nu_e$ as 
\begin{align}
\label{tau_max}
\nu_e \leq \max_{i} (0, \nu_{i,e} - \nu'_{e}(\Delta)) \leq P_e(\Delta).
\end{align}

We define the event $\mathcal{E}_{1,e}(\nu)$ as: 
\begin{equation}
    \mathcal{E}_{1,e}(\nu) = \bigcap_{i \in [N]} \left\{ \nu_{i,e} \leq \nu \right\}
	\label{eq:E1 all agents identified preferences}
\end{equation}
If $\mathcal{E}_{1,e}(\nu)$ holds, then Lemma~\eqref{lem:ucb_lcb} guarantees that conditional on the good event $\mathcal{G}(t_e(\nu))$ all agents have identified their correct preferences by time $\nu$-th occurrence of environment $e$. 

We first define, the quantity $\mathcal{E}_{0}(t) = \exists \sigma_t$ \text{ that satisfies} Eq.~\ref{eq:separation} that inidicates that in round $t$ the top-$N$ rank was successfully recovered. 

Let us also define $\mathcal{E}_{1}(t) = \cap_{e}\{\exists \nu_e: t_e(\nu_e) \leq t \wedge \mathcal{E}_{1,e}(\nu_e)\}$. The event $\mathcal{E}_{1}(t)$ implies the top-$N$ ranks for all the environments are discovered by all the agents on or before time $t$.
\begin{lemma}
Conditioned on $\cap_t \mathcal{G}_t$ the event $\neg\mathcal{E}_{0}(t) \cap \neg\mathcal{E}_{1}(t)$ holds for at most $\min_{\Delta >0} (\tau(\Delta) + \sum_{e}P_e(\Delta))$ rounds.
\label{lem:first_exploration_rounds}
\end{lemma}
\begin{proof}
Let the good event $\cap_t\mathcal{G}(t)$ hold. Also fix a $\Delta >0$. First note that if $\mathcal{T}(t) \geq \tau(\Delta)$ then $\neg\mathcal{E}_{1}(t)$ holds at most an extra $E(1+C)$ rounds due to Eq.~\eqref{tau_max}. However, if we have $\neg\mathcal{E}_{1}(t)$ then in Algorithm~\ref{alg:alg2} we have $|D| < E$ so there can be no  environment discovery due to partial rank match. Furthermore, due to $\neg\mathcal{E}_{0}(t)$ there is no direct environment discovery as well. So there must be an exploration occurring in round $t$.  Therefore, under the event  $\neg\mathcal{E}_{0}(t) \cap \neg\mathcal{E}_{1}(t)$ the number of exploration increases monotonically. So we must have the number of $\neg\mathcal{E}_{0}(t) \cap \neg\mathcal{E}_{1}(t)$ bounded by $\tau(\Delta) + \sum_{e}P_e(\Delta)$. Taking a minium over $\Delta > 0$ concludes the proof.
\end{proof}


\subsection{Characterizing extra exploration needed to have no errors in identification of environment}

Recall that we had defined the following :

\begin{align*}
	\Delta_{i, {\rm rank}}(t) 
	&= \underset{e \neq e(t)}{\min}  \max \{ |{\mu}_{ij}(t) - {\mu}_{ij'}(t)| :(j,j') \in {\rm Inv}(\rho_{i}^e[1:N], \rho_i^{e(t)}[1:N])\}    
\end{align*}

Let ${\rm TopN}_{i,e}$ denote the set of top $N$ arms for agent $i$ under environment $e$ 
$$
\Delta_{\rm min rank} = \min_{1 \leq t \leq T}\hspace{0.5em}\min_{i \in [N]}  \Delta_{i, {\rm rank}}(t).
$$
Observe that 
\begin{align}\label{eq:inversion_pairs}
    {\rm Inv}(pr_i(t), \rho_i^e) \supseteq 
    \Big\{(j,j') &\mid (j,j') \in{\rm Inv}(\rho_{i}^e[1:N], \rho_i^{e(t)}[1:N])
     \land |{\mu}_{ij}(t) - {\mu}_{ij'}(t)| > \Delta_{i,rank}(t) \Big\}
\end{align}

where $pr_i(t)$ is the partial ranking formed using Equation~\eqref{eq:partial}, $\rho_i^e \in \mathbb{R}^K$ is the true preference ranking over all arms by agent $i$ in environment $e$


First, we show that ${\rm Inv}(pr_i(t),\rho_i^{e(t)})=\emptyset$ under the good event. Consider arms $j, j' \in {\rm TopN}_{i,e}$ \ , such that $\mu_{i,j}(t) < \mu_{i,j'}(t)$. First note that due to Lemma~\ref{lem:ucb_lcb} under $pr_{i}(t)$ we have $j \prec j'$ due to condition Eq~\eqref{eq:separation}. Again due to Lemma~\ref{lem:ucb_lcb}, under the good event, for $pr_i(t)$ we either have $j = j'$ or $j \prec j'$. Therefore, for any pair $j, j' \in {\rm TopN}_{i,e}$ we do not have mismatch in $pr_i(t)$, and $\rho_i^e$. So ${\rm Inv}(pr_i(t),\rho_i^e)=\emptyset$.

Now consider the case of $e\neq e(t)$. Consider arm $$(j,j')\in \left\{ \left(j,j'\right) \mid (j,j')\in {\rm Inv}( \rho_i^{e(t)}[1:N], \rho_i^e[1:N] ) \land |{\mu}_{ij}(t) - {\mu}_{ij'}(t)| > \Delta_{i,rank}(t) \right\}.$$  We have the following relation, due to Lemma~\eqref{lemma:del_resolvable},  
\begin{align}
    \Delta_{\text{minrank}} &> 8Ld\sqrt{\frac{\log(t)}{\kappa \explore}}
    \leftarrow \explore > t_{\rm exp} \triangleq\frac{64L^2d^2\log(T)}{\Delta_{\rm min rank}^2}
    \label{max_exploration}
\end{align}
Therefore, when $\explore > t_{\rm exp}$ we can resolve any pair $(j,j')$  with $|{\mu}_{ij}(t) - {\mu}_{ij'}(t)| > \Delta_{i,rank}(t) \geq \Delta_{\text{minrank}}$. Hence, in $pr_i(t)$ we can find a pair $(j,j')$ that is flipped as compared to $\rho_i^e$ when $e\neq e(t)$.
Thus a sufficient condition to ensure that there is no error in the identification of environment at time $t$ onwards is $\explore > t_{\rm exp}$.

Define the event $\mathcal{E}_{2}(t)$ as the event that the partial ranks are recovered by each agent $i$ for the time $t$ 
\begin{equation}
    \mathcal{E}_{2}(t) = \bigcap_{i \in [N], e' \neq e(t)} \left\{ \exists (j,j') : 
    (j,j') \in {\rm Inv}(\ \sigma_{ie(t)}, \rho_i^{e'}[1:N]\ )  \notag \\
    \land |{\mu}_{ij}(t) - {\mu}_{ij'}(t)| > \Delta_{i,rank}(t) \right\}
\end{equation}
Note that the complement refers to a partial match error.
Thus, if $\mathcal{E}_{2}(t)$ does not hold, then there exists an agent such that there are no inversion pairs that can be resolved with the current confidence bounds, and thus that agent is not able to identify the environment. Due to this, all agents will start exploration again from that time step (Exploration Triggering in Algorithm~\eqref{alg:alg2})
\begin{lemma}
Conditioned on $\cap_t \mathcal{G}_t$, the event $\mathcal{E}_{1}(t)\cap  \neg \mathcal{E}_{2}(t)$ holds for at most $g(64L^2d^2\log(T)/\Delta_{\rm min rank}^2)$ rounds.
\label{lemma:additional_exploration_rounds_bound}
\end{lemma}
\begin{proof}
If event $\mathcal{E}_1(t)$ holds, then it means that all agents have already found out the preferences correctly. From Equation~\eqref{max_exploration} we have that for all t such that $\explore > 64L^2d^2\log(T)/\Delta_{\rm min rank}^2 $ all agents would be able to identify the environment with probability one given that they have already found out their preferences correctly. If $\mathcal{E}_2(t)$ holds, then it means that there exists an agent who is not able to identify the environment and $\explore \leq 64L^2d^2\log(T)/\Delta_{\rm min rank}^2 $. In this case, agents will trigger more exploration,
and simultaneously start new exploration phase.
We can bound the number of such rounds as $g(64L^2d^2\log(T)/\Delta_{\rm min rank}^2)$. Recall that $g(x)=2x+\log(2x)$ comes from the fact that extra exploration rounds occur in increasing phase lengths of form $2^a$ and than agents start playing Gale-Shapley algorithm again if the stable match is not found. The $2x$ term comes from loosely upper bounding number of time steps when agents were not able to identify the environment. Also, taking the intersection with $\mathcal{G}(t)$ does not break any of the above arguments, and this is done for the purpose of regret analysis.
\end{proof}

\begin{lemma}
Conditioned on $\cap_t \mathcal{G}_t$, the event $\left\{\exists i \in [N]: m_i(t) \neq m_i^*(e(t))\right\} \cap ((\mathcal{E}_1(t)\cap  \mathcal{E}_2(t)) \cup \mathcal{E}_0(t))$ holds for at most $EN^2$ rounds.
\label{lemma:environment_identification_error}
\end{lemma}
\begin{proof}
If $\mathcal{G}(t) \cap \mathcal{E}_0(t)$  holds then the top-$N$ rank was recovered for round $t$. Moreover, if $\mathcal{G}(t) \cap \mathcal{E}_1(t) \cap  \mathcal{E}_2(t)$ holds, then all agents have identified their preferences and they are able to correctly identify the environment. Thus, there would be no errors in the identification of the environment in both cases. The agents would be playing the Gale-Shapley algorithm for the identified environment, and once the stable matching is found, they will exploit the stable matched arm. Note that there is different stable matching for different environments, and convergence to stable matching in each environment requires  at most $N^2$ rounds as guaranteed by Lemma~\eqref{max_steps_in_gale_shapley}, so a total number of such rounds is bounded by $EN^2$. Also note that for all events $\mathcal{G}(t)\cap \neg((\mathcal{E}_1(t)\cap  \mathcal{E}_2(t)) \cup \mathcal{E}_0(t))$ the algorithm never enters the Gale Shapley phase it is in exploration phase. So there is no cascading error.  
\end{proof}

\subsection{Regret}
\allowdisplaybreaks

Regret of player $i$ by time $T$ is given as 
\begin{align*} 
	\mathbb{E}[R_T^{(i)}]=\mathbb{E}\left[\sum_{t=1}^{T} \left( \mu_{i, m_i^{*}(e(t))}(t) -\mu_{i,m_i(t)}(t) \right)\right],
\end{align*}


\begin{align*}
\mathbb{E}\left[R_i(T)\right] &\leq \sum_{t=1}^{T} \mathbb{E} \left[ \mathbf{1}\{m_i(t) \neq m_i^*(e(t))\} \right] \Delta_{i,\text{max}} \\
&\leq \sum_{t=1}^{T} \mathbb{E} \left[ \mathbf{1}\left\{err(t)\right\} | \cap_{t} \mathcal{G}(t)\right] \Delta_{i,\text{max}} 
+ \sum_{t=1}^{T} \mathbb{E} \left[ \mathbf{1}\{\neg \mathcal{G}(t)\} \right] \Delta_{i,\text{max}}  \quad[err(t) \triangleq m_i(t) \neq m_i^*(e(t))]\\
&\stackrel{(a)}{\leq} \sum_{t=1}^{T} \mathbb{E}\left[\mathbf{1}\left\{err(t) \right\} | \cap_{t} \mathcal{G}(t)\right] \Delta_{i,\text{max}} + \frac{Nd\pi^2}{3} \Delta_{i,\text{max}} \\
&\stackrel{(b1)}{=} \sum_{t=1}^{T} \mathbb{E} \left[ \mathbf{1}\left\{err(t) \cap \neg \mathcal{E}_0(t)\right\} 
+ \mathbf{1}\left\{err(t) \cap \mathcal{E}_0(t)\right\} | \cap_{t}  \mathcal{G}(t) \right] \Delta_{i,\text{max}} + \frac{Nd\pi^2}{3} \Delta_{i,\text{max}} \\
&\stackrel{(b2)}{=} \sum_{t=1}^{T} \mathbb{E} \left[ \mathbf{1}\left\{err(t) \cap \neg \mathcal{E}_0(t) \cap \neg \mathcal{E}_1(t)\right\} +
\mathbf{1}\left\{err(t) \cap \neg \mathcal{E}_0(t) \cap  \mathcal{E}_1(t) \right\}| \cap_{t}  \mathcal{G}(t)\right] \Delta_{i,\text{max}}\\
&+\sum_{t=1}^{T} \mathbb{E}\left[\mathbf{1}\left\{err(t) \cap \mathcal{E}_0(t)\right\} | \cap_{t} \mathcal{G}(t)\right] \Delta_{i,\text{max}} + \frac{Nd\pi^2}{3} \Delta_{i,\text{max}} \\
 &\stackrel{(b3)}{\leq} \underbrace{\sum_{t=1}^{T} \mathbb{E} \left[\mathbf{1}\left\{err(t) \cap \neg \mathcal{E}_0(t) \cap \neg \mathcal{E}_1(t) \right\}| \cap_{t} \mathcal{G}(t)\right]}_{{\text Lemma~\ref{lem:first_exploration_rounds}}} \Delta_{i,\text{max}} \\
 &+ \underbrace{\sum_{t=1}^{T} \mathbb{E} \left[\mathbf{1}\left\{err(t)  \cap  \mathcal{E}_1(t)  \cap\neg \mathcal{E}_2(t)\right\}| \cap_{t} \mathcal{G}(t)\right]}_{{\text Lemma~\ref{lemma:additional_exploration_rounds_bound}}} \Delta_{i,\text{max}}\\
 &+\underbrace{\sum_{t=1}^{T} \mathbb{E}\left[\mathbf{1}\left\{err(t)  \cap \mathcal{E}_0(t)\right\} + \mathbf{1}\left\{err(t)  \cap  \mathcal{E}_1(t) \cap \mathcal{E}_2(t) \right\} | \cap_{t} \mathcal{G}(t)\right]}_{{\text Lemma~\ref{lemma:environment_identification_error}}} \Delta_{i,\text{max}} + \frac{Nd\pi^2}{3} \Delta_{i,\text{max}} \\
&\stackrel{(c)}{\leq}  \left( \min_{\Delta > 0}\left(\frac{64d^2L^2 \log(T)}{\kappa \Delta^2} + \sum_e P_e(\Delta)\right) + g\left(\frac{64L^2d^2\log(T)}{\Delta_{\rm min rank}^2}\right) +  EN^2 + \frac{Nd\pi^2}{3}\right) \Delta_{i,\text{max}} \\
\end{align*}
Step (a) follows from the proof of Lemma (\ref{lem:bad_event_probability}). Step (bi)s follows from the standard probability manipulations. In step (c), we use the Lemmas stated under the terms to bound the specific terms and obtain the final inequality.

\newpage



\end{document}